\documentclass{article}

\usepackage{microtype}
\usepackage{graphicx}
\usepackage{amsfonts}

\usepackage{subcaption}
\usepackage{booktabs} 
\usepackage{enumitem}
\usepackage{hyperref}
\usepackage{siunitx} 
\usepackage{makecell} 
\usepackage{caption} 

\usepackage[accepted]{icml2024}

\usepackage{amsmath}
\usepackage{amssymb}
\usepackage{mathtools}
\usepackage{amsthm}

\usepackage[capitalize,noabbrev]{cleveref}

\theoremstyle{plain}
\newtheorem{theorem}{Theorem}[section]
    \newtheorem{proposition}[theorem]{Proposition}
\newtheorem{lemma}[theorem]{Lemma}

\theoremstyle{definition}
\newtheorem{definition}[theorem]{Definition}

\usepackage{graphicx} 
\usepackage{subcaption} 
\usepackage[textsize=tiny]{todonotes}

\icmltitlerunning{Mitigating Privacy Risk in Membership Inference by Convex-Concave Loss}
        
\begin{document}

\twocolumn[
\icmltitle{Mitigating Privacy Risk in Membership Inference by Convex-Concave Loss}



\begin{icmlauthorlist}
\icmlauthor{Zhenlong Liu}{sustech}
\icmlauthor{Lei Feng}{sutd}
\icmlauthor{Huiping Zhuang}{scut}
\icmlauthor{Xiaofeng Cao}{jilin}
\icmlauthor{Hongxin Wei}{sustech}
\end{icmlauthorlist}

\icmlaffiliation{sustech}{Southern University of Science and Technology}
\icmlaffiliation{scut}{South China University of Technology}
\icmlaffiliation{sutd}{Singapore University of Technology and Design}
\icmlaffiliation{jilin}{Jilin University}

\icmlcorrespondingauthor{Hongxin Wei}{weihx@sustech.edu.cn}

\icmlkeywords{Machine Learning, ICML}

\vskip 0.3in
]



\printAffiliationsAndNotice{}
\begin{abstract}

Machine learning models are susceptible to membership inference attacks (MIAs), which aim to infer whether a sample is in the training set. 
Existing work utilizes gradient ascent to enlarge the loss variance of training data, alleviating the privacy risk.
However, optimizing toward a reverse direction may cause the model parameters to oscillate near local minima, leading to instability and suboptimal performance.
In this work, we propose a novel method -- Convex-Concave Loss (\textbf{CCL}), which enables a high variance of training loss distribution by gradient descent.
Our method is motivated by the theoretical analysis that convex losses tend to decrease the loss variance during training.
Thus, our key idea behind CCL is to reduce the convexity of loss functions with a concave term.
Trained with CCL, neural networks produce losses with high variance for training data, reinforcing the defense against MIAs. 
Extensive experiments demonstrate the superiority of CCL, achieving a state-of-the-art balance in the privacy-utility trade-off.
\end{abstract}

\section{Introduction}
\label{sec:intro}

Deep Neural Networks (DNNs) have achieved tremendous performance for various learning tasks, with sufficient capacity~\cite{ijcai2021p467}. The powerful capability enables models to memorize information of training data~\cite{Zhang2017}, therefore being highly susceptible to membership inference attacks (MIAs) \cite{shokri2017membership}. In particular, membership inference attacks are designed to infer whether a sample is included in the training set of a target model. Such attacks can increase the risks of violating privacy regulations, making it challenging to apply machine learning techniques in sensitive applications, like health care~\cite{paul2021defending}, financial service~\cite{mahalle2018data}, and DNA sequence analysis~\cite{arshad2021analysis}. This gives rise to the importance of designing robust algorithms to secure DNNs from MIAs.


In the literature, sample loss has been a representative metric of membership inference attacks. In particular, a large gap in the expected loss values on the member (training) and non-member data is proved to be sufficient for attacks \cite{yeom2018privacy}. Consequently, defenders can mitigate the privacy risk by reducing distinguishability between the member and non-member loss distributions. 
Recently, RelaxLoss~\cite{chen2022relaxloss} utilizes gradient ascent to promote a high variance of training loss distribution, which is shown to strengthen the defense against MIAs. 
However, optimizing toward a reverse direction may cause the model parameters to oscillate near local minima, leading to instability and suboptimal performance.
This motivates our method, which enables us to enlarge the training loss variance via gradient descent.



 In this work, we propose a novel and generalized method -- Convex-Concave Loss (\textbf{CCL}), by integrating a concave term into convex losses. Our method is motivated by a theoretical analysis of the connection between the convexity of loss functions and the resulting loss variance. We demonstrate that convex loss functions are optimized to encourage a small loss variance (see Theorem~\ref{convex loss}), being vulnerable to membership inference attacks. On the contrary, concave functions can increase the loss variance during gradient descent, which are expected to mitigate privacy risk.

 
Thus, our key idea behind CCL is to decrease the convexity of loss functions for a large loss variance.
This can be achieved by incorporating a concave term into the original convex losses, e.g., cross-entropy loss.
In effect, the resulting loss weakens the convexity of the original convex loss at the late stage of training and can converge to the optimum of the convex loss.
Trained with CCL, the network tends to produce losses with high variance for training data, reducing the differentiability of sample losses between the member and non-member data.



To verify the effectiveness of our method, we conduct extensive evaluations on five datasets, including Texas100 \cite{TexasInpatient2006}, Purchase100~\cite{KaggleValuedShoppers2014}, CIFAR-10/100~\cite{krizhevsky2009learning}, and ImageNet~\cite{russakovsky2015imagenet} datasets. The results demonstrate our methods can improve utility-privacy trade-offs across a variety of attacks based on neural network, metric, and data augmentation. For example, our method formulated using a concave quadratic function, significantly diminishes the attack advantage in loss-metric-based from 29.67\% to 18.40\% - a relative reduction of 62.01\% in privacy risk, whilst preserving the test accuracy not worse than the vanilla model. Our code is available at \href{https://github.com/ml-stat-Sustech/ConvexConcaveLoss}{https://github.com/ml-stat-Sustech/ConvexConcaveLoss}.

Our contributions are summarized as follows:

\begin{enumerate}
    \item We introduce the concept of Convex-Concave Loss (CCL), a generalized loss function that incorporates a concave term into the original convex loss, i.e., Cross-Entropy (CE) loss. This approach stands as a novel and effective countermeasure against MIAs.
    \item We provide rigorous theoretical analyses to establish a key insight: convex loss functions tend to decrease the loss variance. In contrast, concave functions can enlarge the variance of the training loss distribution.
    \item We establish that CCL offers a state-of-the-art balance in the privacy-utility trade-off, with extensive experiments on Texas100, Purchase100, CIFAR-10/100, and ImageNet datasets with diverse model architectures.
\end{enumerate}

\section{Background}
\paragraph{Setup.} In this paper, we study the problem of membership inference attacks in $K$-class classification tasks. Let the feature space be $\mathcal{X} \subset \mathbb{R}^{d}$ and the label space be $\mathcal{Y} = \{ 1,\dots,K\}$. 
Let us denote by $(\boldsymbol{x}, y) \in (\mathcal{X} \times \mathcal{Y})$ an example containing an instance $\boldsymbol{x}$ and a real-valued label $y$. 
Given a training dataset $\mathcal{S} = \{(\boldsymbol{x_n}, y_n) \}_{i=1}^N$ \textit{i.i.d.} sampled from the data distribution $\mathcal{D}$, our goal is to learn a model $h_S \in \mathcal{H}$, that minimizes the following expected risk:
\begin{align} \label{expected_risk}
    R(\mathcal{L}) = \mathbb{E}_{(\boldsymbol{x},y) \sim \mathcal{D}} [\mathcal{L}(h(\boldsymbol{x}),y)]
\end{align}
where $\mathbb{E}_{(\boldsymbol{x},y) \sim \mathcal{D}}$ denotes the expectation over the data distribution $\mathcal{D}$ and $\mathcal{L}$ is a conventional loss function (such as cross-entropy loss) for classification.


\paragraph{Membership Inference Attacks.} Given a data point $(\boldsymbol{x}, y)$ and a trained target model $h_\mathcal{S}$, attackers aim to identify if $(\boldsymbol{x}, y)$ is one of the members in the training set $\mathcal{S}$, which is called membership inference attacks (MIAs) \cite{shokri2017membership, yeom2018privacy, salem2018ml}. In MIAs, it is generally assumed that attackers can query the model predictions $h_S(\boldsymbol{x})$ for any instance $\boldsymbol{x}$. Here, we focus on standard black-box attacks \cite{truex2018towards}, where attackers can access the knowledge of model architecture and the data distribution $\mathcal{D}$. 

In the process of attack, the attacker has access to a query set $\mathcal{Q}=\left\{\left(\boldsymbol{z}_i, m_i\right)\right\}_{i=1}^J$, where $\boldsymbol{z}_i$ denotes the $i$th data point $(\boldsymbol{x}_i, y_i)$ and $m$ is the membership attribute of the given data point $(\boldsymbol{x}_i, y_i)$ in the training dataset $\mathcal{S}$, i.e., $m_i = \mathbb{I}[(\boldsymbol{x}_i, y_i) \in \mathcal{S}]$. In particular, the query set $\mathcal{Q}$ contains both member (training) and non-member samples, drawn from the data distribution $\mathcal{D}$. Then, the attacker $\mathcal{A}$ can be formulated as a binary classifier, which predicts $m_i \in \{0,1\}$ for a given example $(\boldsymbol{x}_i, y_i)$ and a target model $h_\mathcal{S}$.
To quantify the performance of the attack model $\mathcal{A}$, we use the \textit{membership advantage}~\cite{yeom2018privacy} :
\begin{align} \label{def: adv}
    \mathit{Adv}(\mathcal{A}) &:= \operatorname{Pr}(\mathcal{A}(h_\mathcal{S}(\boldsymbol{x}),y)=1 | m=1) \notag \\
    &\quad - \operatorname{Pr}(\mathcal{A}(h_\mathcal{S}(\boldsymbol{x}),y)=1 | m=0) \\
    &= 2\operatorname{Pr}(\mathcal{A}(h_\mathcal{S}(\boldsymbol{x}),y) = m )-1 \notag
\end{align}
Equivalently, $\mathit{Adv}(\mathcal{A})$ can be seen as the difference between $\mathcal{A}$'s true and false positive rates.
\paragraph{Loss variance.} In the literature, Theorem 3 in \citep{yeom2018privacy} shows that, for regression tasks, the \textit{membership advantage} can be approximated as $\mathrm{erf}(1 \slash \sqrt{2})-\mathrm{erf}(\sigma_{\mathcal{S}} \slash \sqrt{2} \sigma_{\mathcal{D}})$, where $\mathrm{erf}(x) = \frac{1}{\sqrt{\pi}} \int_{-z}^{x} \exp(-t^2) dt$, $\sigma_{\mathcal{S}}$ and $\sigma_{\mathcal{D}}$ denote the loss variances of member and non-member data over examples, respectively. This suggests that increasing the loss variance of training data $\sigma_{\mathcal{S}}$ can help decrease the membership advantage, i.e., enhance the defense against MIAs. We formally discuss the effect of loss variance on the attack advantage in Section~\ref{sec:discussion}.

Recently, RelaxLoss \cite{chen2022relaxloss} applies gradient ascent to increase the loss variance of the training data, thereby alleviating the privacy risk.
However, optimizing toward a reverse direction may cause the model parameters to oscillate near local minima, preventing convergence to the global optimum. Consequently, the inconsistency of optimizing directions over iterations will result in the suboptimal performance of the trained model in utility (see Figure~\ref{fig:cifar10_baselines}). This motivates us to design a loss function that increases $\sigma_{S}$ during gradient descent.
\section{Theoretical Motivation}
In this section, we begin with a formal analysis to show that cross-entropy loss tends to reduce $\sigma_{\mathcal
S}$ due to its convexity. Based on this, we propose concave functions, which are theoretically shown to increase $\sigma_{\mathcal{S}}$. 

For a sample $\boldsymbol{x} \in \mathcal{X}$, we denote the distribution over different labels by $q(k|\boldsymbol{x})$, the output probability of $h_S(\boldsymbol{x})$ by $p(k|\boldsymbol{x})$. For simplicity, we denote $p_k$, $q_k$ as abbreviations for $p(k|\boldsymbol{x})$ and $q(k|\boldsymbol{x})$, respectively. In particular, the confidence in the true label $p(y|\boldsymbol{x})$ is abbreviated as $p_y$. 
\subsection{Convex function decreases the loss variance} 
Here, we provide a formal analysis to show how the loss function influences the loss variance of training data. Given a certain model, we can view $p_y = p(y|\boldsymbol{x})$ as a random variable, which depends on the pair of random variables $(\boldsymbol{x}, y) \sim \mathcal{D}$.
Then, the training objective (\ref{expected_risk}) with cross-entropy loss $\ell_{\mathrm{ce}}$ can be rewritten as:
\begin{align*}
    \min_{h} \quad \mathbb{E}_{\mathcal{D}} [-\log p_y]
\end{align*}
where $\mathbb{E}_{\mathcal{D}}$ denotes the expectation over the data distribution $\mathcal{D}$.
Let $1-\epsilon$ and $\sigma^2$ be the mean and variance of $p_y$ over the data distribution $\mathcal{D}$, where $0<\epsilon <1$.
By Taylor expansion, we have 
\begin{align*}
\mathbb{E}_{\mathcal{D}}(-\log p_y)  &\geqslant 
\mathbb{E}_{\mathcal{D}} [(1-p_y)+ \frac{1}{2}(1-p_y)^2]
\\ &= \mathbb{E}_{\mathcal{D}} [(1-p_y)] + \frac{1}{2} \mathbb{E}_{\mathcal{D}}[(1-p_y)^2]
\\ &= \epsilon + \frac{1}{2}(\sigma^2+\epsilon^2)
\end{align*}
Thus, we obtain a lower bound for the expected value of training loss, which depends on $\epsilon$ and $\sigma^2$. It implies that the training loss can be optimized toward a smaller value of variance $\sigma^2$, corresponding to a smaller loss variance \footnote{The monotonic relationship between loss variance and $\sigma^2$ is proved in Appendix~\ref{app:variance_connection}.}.



Note that the above property of cross-entropy loss is dependent on the positive coefficient of $\sigma^2$, which can be calculated from the second-order derivation of cross-entropy loss. In other words, the relationship between cross-entropy loss and loss variance stems from its convexity. This insight can be formalized as follows:
\begin{theorem} \label{convex loss}
Given a twice continuously differentiable function $\ell \in C^2(0,1]$ such that $\ell(1) = 0$ and $\ell'(x) < 0, \forall x \in (0,1]$. If $\ell$ is strictly \textbf{convex}, then
\begin{align*} 
    \mathbb{E}_{\mathcal{D}}[\ell(p_y)]  \geqslant A\epsilon + \frac{B}{2}(\epsilon^2 + \sigma^2)
\end{align*}
where $A = -\ell^{\prime} (1)>0$, $B \geqslant 0$ is a non-negative lower bound of $\ell^{\prime \prime}(x)$. 
\end{theorem}

The detailed proof is presented in the Appendix \ref{proof for convex loss}.
By Theorem \ref{convex loss}, we show that the coefficient of $\sigma^2$ is positive if the loss function is strictly convex with respect to $p_y$. Similar to cross-entropy loss, such convex loss functions will be optimized to encourage a smaller loss variance, being vulnerable to membership inference attacks. 

\begin{figure}[t] 
    \centering
    \begin{subfigure}[b]{0.24\textwidth} 
        \centering
        \includegraphics[width=\textwidth]{\detokenize{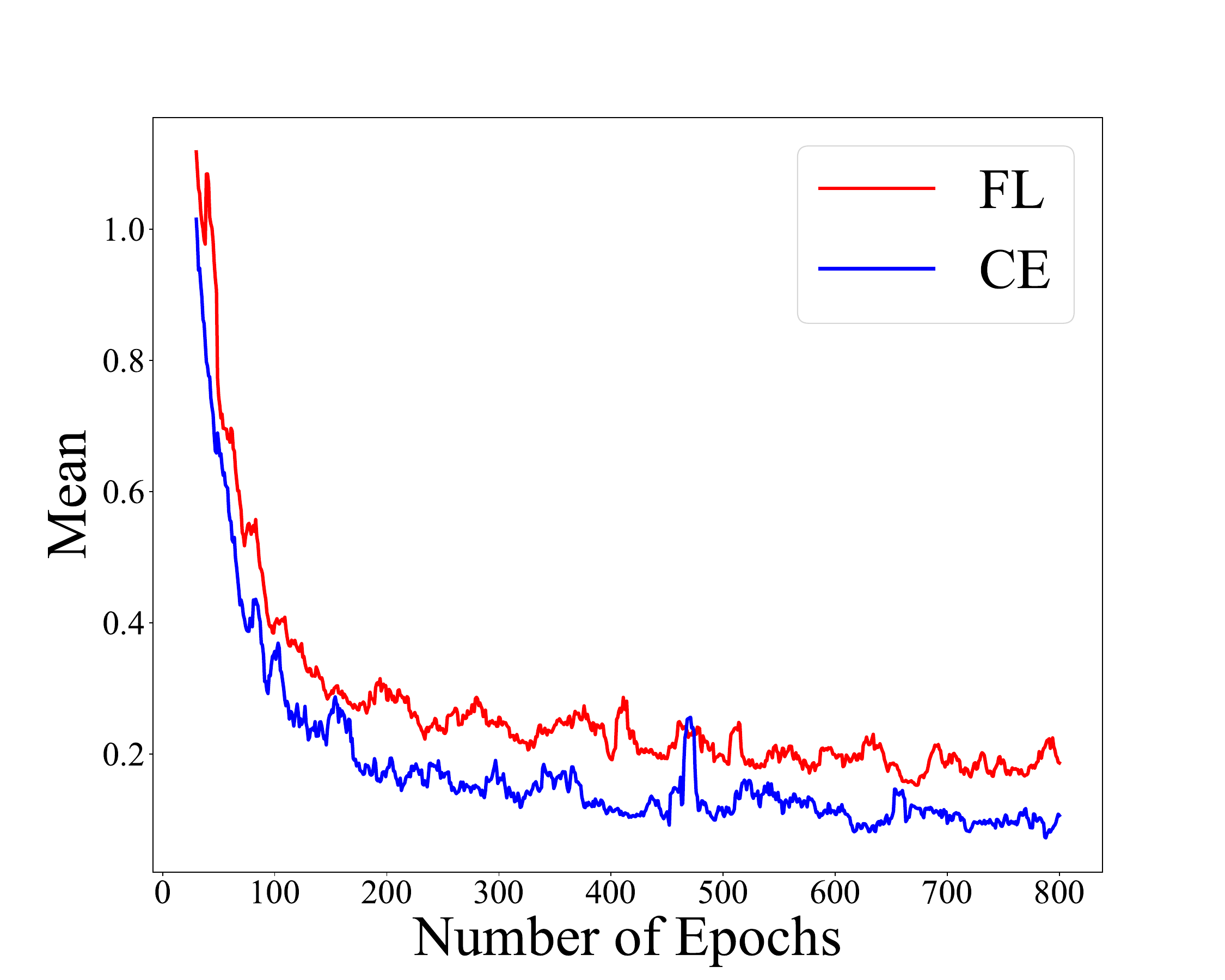}}
        \caption{Mean of Loss}
        \label{fig:cifar10_mean_sta_epoch_smoothed}
    \end{subfigure}%
    \hfill 
    \begin{subfigure}[b]{0.24\textwidth} 
        \centering
        \includegraphics[width=\textwidth]{\detokenize{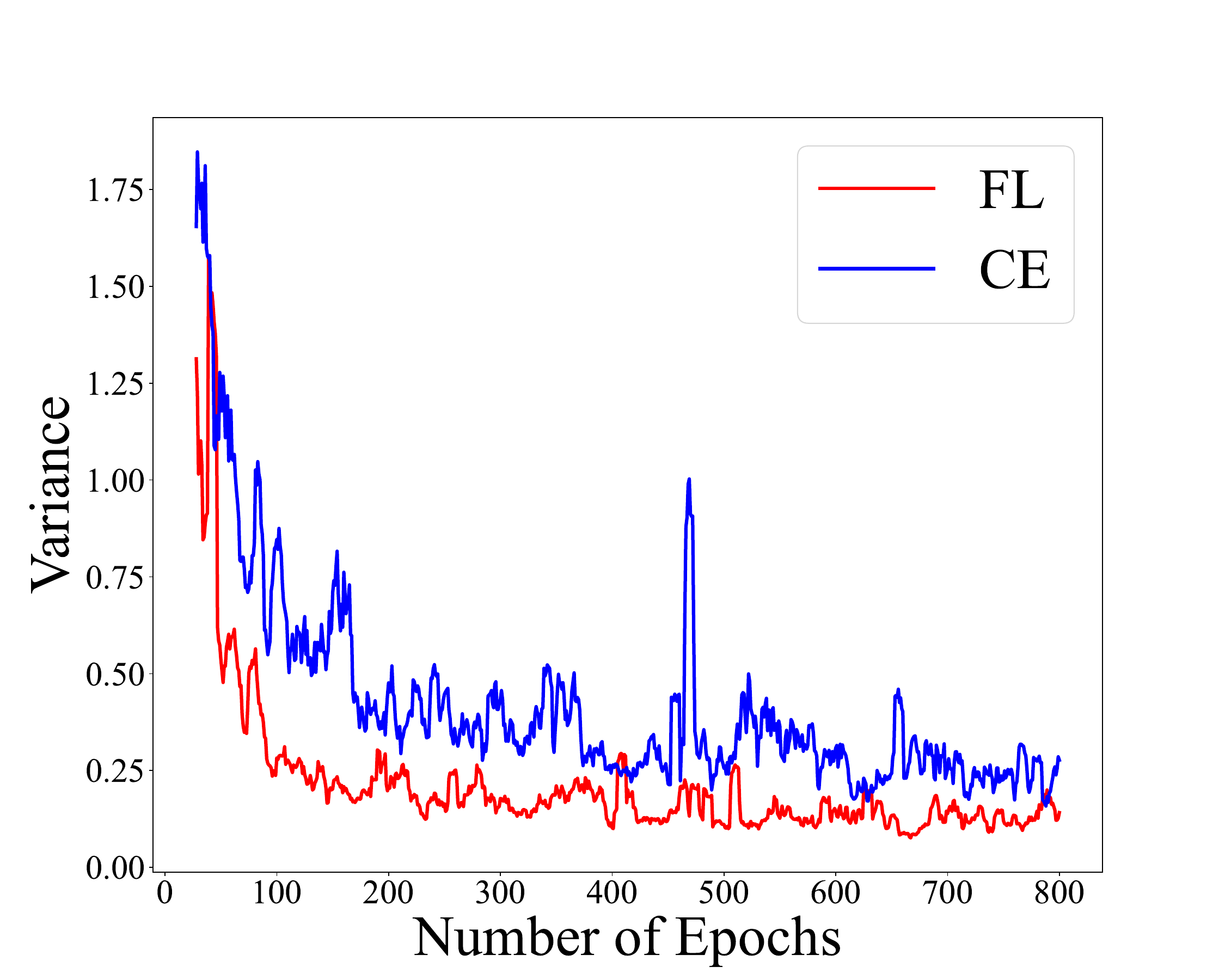}}
        \caption{Variance of Loss}
        \label{fig:cifar10_variance_sta_epoch_smoothed}
    \end{subfigure}
    \caption{The mean and variance of loss under different epochs. Models are trained on CIFAR-10 with Resnet-34 using Cross-entropy loss (CE) and Focal loss (FL).}
    \label{fig:focal_ce_sta}
\end{figure}

To provide a straightforward view, we empirically verify the connection between convexity and loss variance. In particular, we introduce Focal Loss \cite{lin2017focal}: $\ell_{\mathrm{fl}} = - (1-p_y)^{\gamma} \log(1-p_y)$ with $\gamma =2$, which exhibits stronger convexity than cross-entropy loss, i.e., $\ell_{\mathrm{fl}}^{\prime \prime} \geqslant \ell_{\mathrm{ce}}^{\prime \prime}, \forall x \in (0,1]$. 
In Figure \ref{fig:focal_ce_sta}, we plot the dynamics of mean and variance of losses for the training data during the training.
Indeed, Focal loss results in a smaller loss variance than cross-entropy loss, with a higher mean of losses. In this way, we verified that the convexity of loss functions contributes to the decrease of loss variance, which may exacerbate the privacy risk of neural networks. We proceed by exploring concave functions, targeting this problem.


\subsection{Can concave functions increase the loss variance?}
In our previous analysis, we show that convex loss can enlarge the privacy risk with a small loss variance. Conversely, we explore the properties of concave functions in the same setting, as shown below.

\begin{theorem}
    \label{concave loss}
    Given a twice continuously differentiable function $\ell \in C^2[0,1]$ such that $\ell(1) = 0$ and $\ell'(x) < 0, \forall x \in [0,1]$. If $\ell$ is strictly \textbf{concave}, there must exist a \textbf{negative} constant $B\leqslant0$ such that
\begin{align} 
\mathbb{E}_{\mathcal{D}}[\ell(p_y)]  = A\epsilon + B(\sigma^2+\epsilon^2)
\end{align}

where \( A = -\ell'(1)  > 0 \).


\end{theorem}

The detailed proof is presented in the Appendix \ref{proof for concave loss}. Given the above theorem, we find that the coefficient of $\sigma^2$ is \textbf{negative} if the loss function $\ell$ is strictly concave. In other words, a concave loss function can increase the loss variance during gradient descent, which is the expected property for mitigating privacy risk. In what follows, we propose a general framework that endows the robustness to cross-entropy loss against membership inference attacks.

\section{Our Proposed Method}
 Theorem \ref{concave loss} indicates that concave functions can be leveraged to design loss functions. To design our loss function, we first introduce a formal definition of the concave term.
\begin{definition}[Concave Term] \label{Concave Term} We define a concave function set as:
$$\mathcal{F} = \{ f \in C^{2}[0, 1] \mid f'(x) < 0, f''(x) < 0,  \forall x \in [0, 1] \}$$
\end{definition}

In this definition, (1)$f'(x) < 0$ ensures that the smaller the objective loss, the larger the confidence of true label $p_y$, which can preserve utility, (2) $f''(x)<0$ ensures that this a concave function that can help increase the variance of $p_y$. 
\paragraph{Convex-concave loss.}
We propose to add a concave term into the original loss function (e.g., cross-entropy loss), which is called \emph{Convex-Concave Loss} (CCL):
\begin{equation} \label{ConcaveLoss}
    \ell_{\mathrm{ccl}}  = \alpha \hat{\ell} + (1-\alpha) \tilde{\ell}
\end{equation}

where $\hat{\ell}$ is the origin convex function, $\tilde{\ell} \in \mathcal{F}$ is a concave term, and $\alpha \in [0,1]$ denotes a hyperparameter to adjust the trade-off between privacy and utility flexibly. In this paper, we just consider cross-entropy loss as the original loss function. The experiments with other convex loss functions are provided in Appendix \ref{other convex functions}.

 Theorem \ref{concave loss} indicates that concave functions can be leveraged to design loss functions. However, in the early stages of training epochs, this approach tends to result in a smaller step size in the gradient ascent process. Consequently, we employ CE loss to facilitate better convergence \cite{zhang2018generalized} and more effective learning. 

For specific concave functions, it is possible to select commonly used monotonic non-linear functions for their design. For instance, we can take the negative of the exponential function as Concave Exponential Loss (CEL):
\begin{align}
    \tilde{\ell}_{\text{exp}} &= -\exp(p_y) 
\end{align}
Alternatively, we can just employ a quadratic polynomial function as Concave Quadratic Loss (CQL):
\begin{equation}
    \tilde{\ell}_{\text{qua}} =-p_y -\frac{1}{2}p_y^2 
\end{equation} 
These two concave terms are all in $\mathcal{F}$. We denote these two concave functions integrated with $\ell_{\mathrm{ce}}$ as CCEL and CCQL.

\paragraph{Gradient analysis.} \label{Gradient analysis} 
There we derive the gradients of CCL. Consider the case of a single true label, we obtain the gradient of the concave term $\tilde{l} \in \mathcal{F}$ w.r.t the logits $z_j$ as follows: 
\begin{align}
   \frac{\partial \tilde{\ell}}{\partial z_j} &= \frac{\partial p_y}{\partial z_j} \tilde{\ell}^{\prime}(p_y) \\
 &= 
\begin{cases} \label{gradient analysis concave}
 p_y(1-p_y) \cdot \tilde{\ell}^{\prime}(p_y) \leqslant 0, & j =y\\
 -p_j p_y \cdot \tilde{\ell}^{\prime}(p_y) \geqslant 0, & \text{otherwise}
\end{cases}
\end{align}

As for CE loss $\ell_{\mathrm{ce}}$, the gradient is 
\begin{equation}
\frac{\partial \ell_{\mathrm{ce}}}{\partial z_j}= 
\begin{cases}p_y-1 \leqslant 0, & j=y 
\\ p_j \geqslant 0, & \text{otherwise}
\end{cases}
\end{equation}
It is clear that the gradient of the concave term in $\mathcal{F}$ has the same sign with CE loss for each $j$, so does the complete loss function $\ell_{\mathrm{ccl}}$:
\begin{align}
\frac{\partial \ell_{\mathrm{ccl}}}{\partial z_j} &= 
\begin{cases}(p_y-1)[\alpha- (1-\alpha) p_y\tilde{\ell}^{'}(p_y)] \leqslant 0, & j=y 
\\ p_j [\alpha- (1-\alpha) p_y\tilde{\ell}^{\prime}(p_y)] \geqslant 0, & \text{otherwise}
\end{cases} \\
&=  [\alpha- (1-\alpha) p_y\tilde{\ell}^{\prime}(p_y)]\frac{\partial \ell_{\mathrm{ce}}}{\partial z_j}
\end{align}

where $\alpha- (1-\alpha) p_y\tilde{\ell}^{\prime}(p_y) \geqslant 0$.
Considering the concavity of $\tilde{\ell}$, it follows that a greater $p_y$ will result in a larger magnitude of $| p_y\tilde{\ell}^{\prime}(p_y) |$. This can be interpreted such that $\alpha - (1-\alpha) p_y\tilde{\ell}^{\prime}(p_y)$ acts as an acceleration coefficient. That is, for these samples with higher $p_y$, this coefficient serves to increase $p_y$ more rapidly. Consequently, this leads to a broadening in the range of confidence distributions.

Furthermore, we provide bounds of the gradient of $\ell_{\mathrm{ccl}}$.
\begin{proposition} \label{bound}
For any input $\boldsymbol{x}$ and any $\alpha >0$, the gradient of $\ell_{\mathrm{ccl}}$ w.r.t logits $z_j$ is bounded above and below as follows:
    \begin{equation}
   \alpha  \frac{\partial \ell_{\mathrm{ce}} }{\partial z_j}
   \leqslant  
   \frac{\partial \ell_{\mathrm{ccl}} }{\partial z_j}
   \leqslant  
   [\alpha + A (1-\alpha)]\frac{\partial \ell_{\mathrm{ce}}}{\partial z_j}
\end{equation}
where $A = -\tilde{\ell}^{\prime} (1)>0$.
\end{proposition}

From Proposition 4.2, the gradient of $\ell_{\mathrm{ccl}}$ is bounded by the scaled gradient of $\ell_{\mathrm{ce}}$ in each dimension. Therefore, the gradient of the proposed loss $\frac{\partial \ell_{\mathrm{ccl}}}{\partial z_j}$ will approach zero when CE loss achieves its optimum ($\frac{\partial \ell_{\mathrm{ce}} }{\partial z_j}=0$). This indicates that our approach is capable of converging in the same region where CE achieves its optimum.

\section{Experiments}
\begin{figure*}[!ht]
    \centering
    \includegraphics[width=0.9\linewidth]{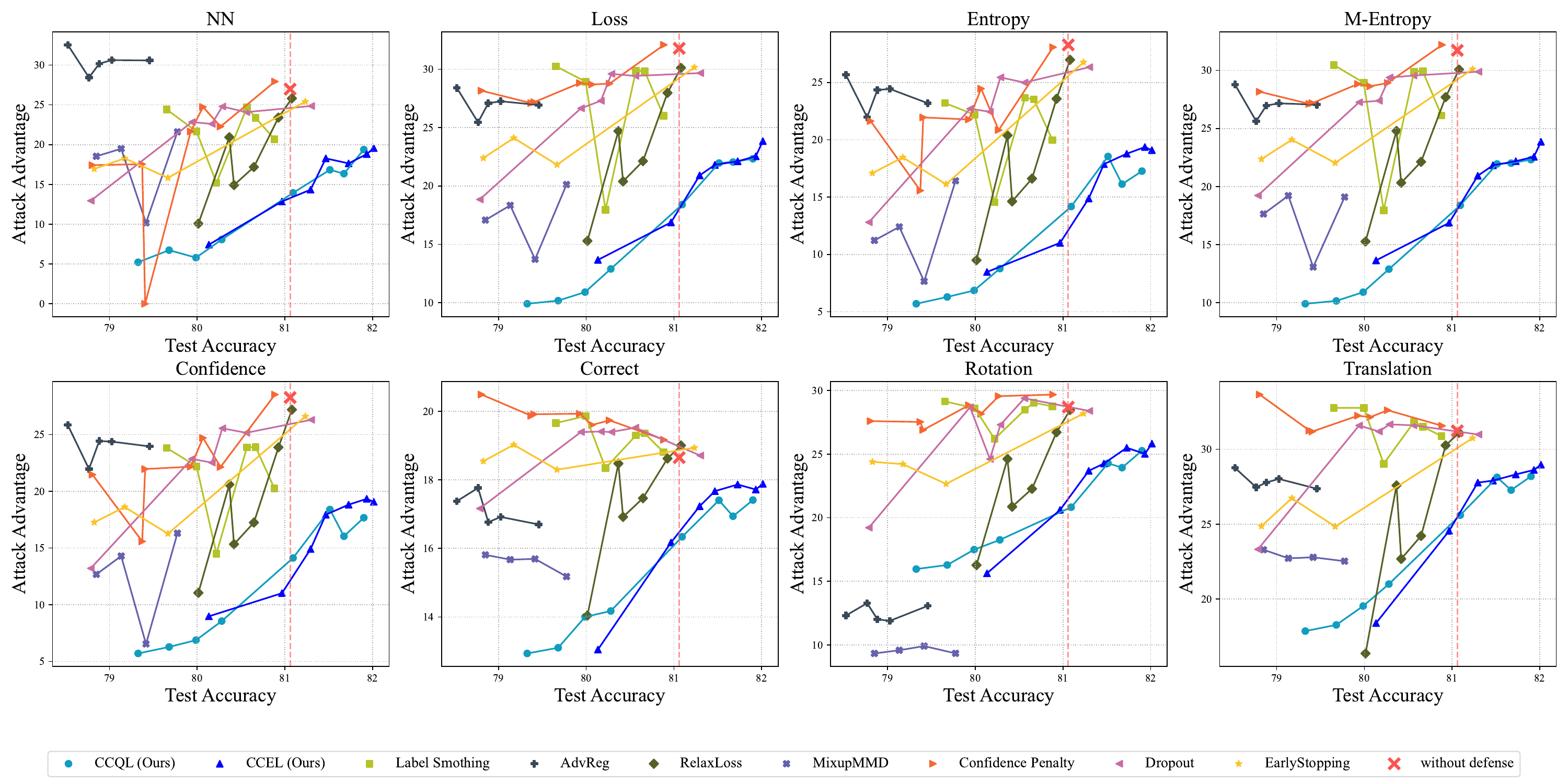}
    \caption{Comparisons of seven defense mechanisms on CIFAR-10 dataset utilizing  Resnet34 architecture. Each subplot is allocated to a distinct attack method, wherein individual curves represent the performance of a defense mechanism under different hyperparameter settings. The horizontal axis represents the target models' test accuracy (the higher the better), and the vertical axis represents the corresponding attack advantage (defined in Definition \ref{def: adv}, the lower the better). To underscore the disparity between the defense methods and the vanilla (undefended model), we plot the dotted line originating from the vanilla results.}
    \label{fig:cifar10_baselines}
\end{figure*}
In this section, we validate the effectiveness of our CCL across a wide range of datasets with diverse models, various attack models, and multiple defense baselines.
\subsection{Setups}
\paragraph{Datasets.}
In our evaluation, we employ five datasets: Texas100 \cite{TexasInpatient2006}, Purchase100 \cite{KaggleValuedShoppers2014}, CIFAR-10, CIFAR-100 \cite{krizhevsky2009learning}, and ImageNet \cite{russakovsky2015imagenet}. For standard training methods, we split each dataset into four subsets, with each subset serving alternately as the training or testing set for the target and shadow models. As for adversarial training algorithms that incorporate adversary loss—such as Mixup+MMD \cite{li2021mixup} and adversarial regularization \cite{nasr2018machine}—we divide the datasets into five subsets. The additional subset is specifically utilized to generate adversary loss.
\paragraph{Training details.}
\begin{figure*}
    \centering
    \includegraphics[width=0.9\linewidth]{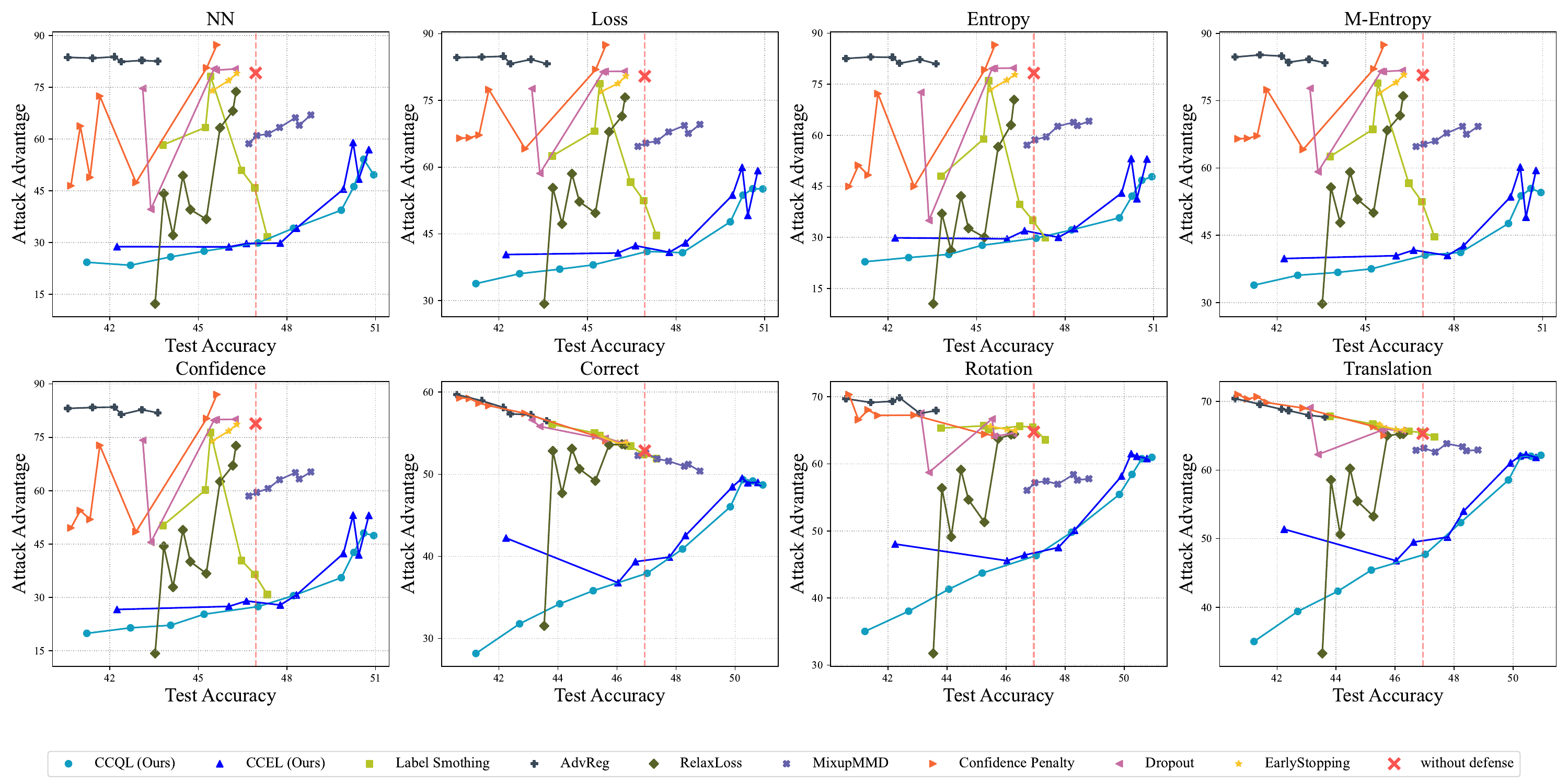}
    \caption{Comparisons of seven defense mechanisms on CIFAR-100 dataset utilizing  Densenet121 architecture. Each subplot is allocated to a distinct attack method, wherein individual curves represent the performance of a defense mechanism under different hyperparameter settings. The horizontal axis represents the target models' test accuracy (the higher the better), and the vertical axis represents the corresponding attack advantage (defined in Definition \ref{def: adv}, the lower the better). To underscore the disparity between the defense methods and the vanilla (undefended model), we plot the dotted line originating from the vanilla results.} 
    \label{fig:cifar100_baselines}
\end{figure*}
We train the models using SGD with a momentum of 0.9, a weight decay of 0.0005, and a batch size of 128. We set the initial learning rate as 0.1 and drop it by a factor of 10 at each decay epoch. For CIFAR-10 and CIFAR-100, we conduct training using a 34-layer ResNet \cite{he2016deep} and a 121-layer DenseNet \cite{huang2017densely}, with decay milestones set at {150, 225} over a total of 300 epochs. In the case of Imagenet, we employ a 121-layer DenseNet with decay milestones at \{30, 60\}, spanning a total of 90 epochs. For Texas100 and Purchase100, training is performed using MLPs as described in previous studies \cite{nasr2018machine, jia2019memguard}, with decay milestones at \{50, 100\} across 120 total epochs. 


\subsection{Hyperparameter Tuning}
In our approach to hyperparameter tuning, we align with the protocols established by previous work \cite{chen2022relaxloss}. In particular, we employ hyperparameter tuning focused on a single hyperparameter, \(\alpha\) defined in \ref{ConcaveLoss}. Through a detailed grid search on a validation set, we adjust \(\alpha\) to achieve an optimal balance. This process involves evaluating the privacy-utility implications at various levels of \(\alpha\) and then selecting the value that aligns with our specific privacy/utility objectives, thereby enabling precise management of the model's privacy and utility. Following focal loss \cite{lin2017focal}, we set a scalar factor on our loss functions. Specifically, for CIFAR-10, the scale factor is 0.01; for CIFAR-100, it is 0.05. For other datasets, we set it as 1. As for our hyperparameter $\alpha$, we vary it across $\{0.1, \dots ,0.9\}$.

\paragraph{Attack methods.} \label{attack method}
In our study, we experiment with three classes of MIA: (1) Neural Network-based Attack (NN) \cite{shokri2017membership, hu2022membership}, which leverages the full logits prediction as input for attacking the neural network model. (2) Metric-based Attack, employing specific metrics computation followed by a comparison with a preset threshold to ascertain the data record's membership status. The metrics we chose for our experiments include Correctness, Loss \cite{yeom2018privacy}, Confidence, Entropy \cite{salem2018ml}, and Modified-Entropy (M-entropy) \cite{song2021systematic}. (3) Augmentation-based Attack \cite{choquette2021label}, utilizing prediction data derived through data augmentation techniques as inputs for a binary classifier model. In this category, we specifically implemented rotation and translation augmentations. 

For the details of the attack, we assume the most powerful black-box adaptive attack scenario: the adversary has complete knowledge of our defense mechanism and selected hyperparameters. To implement this, we train shadow models with the same settings used for our target models.

\paragraph{Defense baselines.}
We compare CCL with seven defense methods: RelaxLoss \cite{chen2022relaxloss}, Mixup+MMD \cite{li2021mixup}, Adversary regularization (Adv-Reg) \cite{li2021membership}, Dropout~\cite{srivastava2014dropout}, Label Smoothing \cite{guo2017calibration}, Confidence Penalty \cite{pereyra2017regularizing}, and Early Stopping \cite{yao2007early}. 
\begin{table*}[!ht]
    \centering
    \begin{tabular}{@{}cSSSSS@{}}
    \toprule
        \multicolumn{1}{c}{Dataset} & {Texas} & {Purchase} & {ImageNet} & {CIFAR-10} & {CIFAR-100} \\ 
    \midrule
        CCQL & 0.607 & \textbf{0.868} & \textbf{0.610} & 0.769 & \textbf{0.487} \\
        CCEL & \textbf{0.608} & 0.864 & 0.609 & \textbf{0.797}& 0.480 \\
        w/o & 0.551 & 0.858 & 0.598 & 0.741 & 0.273 \\
    \bottomrule
    \end{tabular}
    \caption{P1 score (defined in Equation \ref{p1_score}) evaluated on target models trained on different datasets. The bold indicates the best results. Here, "w/o" denotes undefended models.} 
    \label{p1_datasets}
\end{table*}
\paragraph{Evaluation Metrics.}
To comprehensively assess our method's impact on privacy and utility, we employ three evaluation metrics that encapsulate utility, privacy, and the balance between the two. Utility is gauged by the test accuracy of the target model. Privacy is measured through the attack advantage, as defined in Equation \ref{def: adv}. To assess the trade-off between utility and privacy, we utilize the P1 score \cite{paul2021defending}, which is defined as:
\begin{equation} \label{p1_score}
    P1 = 2 * \frac{\mathit{Acc} * (1-\mathit{Adv})}{(\mathit{Acc}) +(1 -\mathit{Adv})}
\end{equation}
where $\mathit{Acc}$ denotes the test accuracy and $\mathit{Adv}$ denotes the attack advantage of the attacker on the target model. 
\subsection{Results}
\paragraph{Can CCL improve privacy-utility trade-off ?}
In Picture \ref{fig:cifar10_baselines} and Picture \ref{fig:cifar100_baselines}, we plot privacy-utility curves to show the privacy-utility trade-off. The horizontal axis represents the performance of the target model, and the vertical axis represents the attack advantage defined in  \ref{def: adv}.
A salient observation is that both of our methods drastically improve the privacy-utility trade-off. In particular, for these points that perform better than vanilla for utility (the area to the right of the dotted line), the privacy-utility curves of our methods are always below those of others. This means we can always obtain the highest privacy for any utility requirement higher than the undefended model. 
 For example, on the CIFAR10 dataset, we focus on the hyperparameter $\alpha$ corresponding to the model with the lowest attack advantage with the constrain condition that test accuracy is better than vanilla, then our method with quadratic function can decrease the attack advantage of loss-metric-based from 29.67\% to 18.40\% compared with Dropout (the most powerful defend method under our condition above).
\paragraph{Is CCL effective with different datasets?}
To ascertain the efficacy of our proposed method across heterogeneous data, we have executed a series of experiments on a diverse array of datasets, encompassing tabular and image datasets. For the experimental results shown in Table \ref{p1_datasets}, we have set the adjustment coefficient of CCL to a constant value, specifically $\alpha = 0.5$. To assess the privacy-utility balanced performance, we use the highest attack advantage of all attack methods to calculate the P1 score 
From the results, we observe that both of our methods yield a consistent improvement in the P1 score. 
\begin{figure*}[!ht] 
    \centering
    \begin{subfigure}[b]{0.30\textwidth} 
        \centering
        \includegraphics[width=\textwidth]{\detokenize{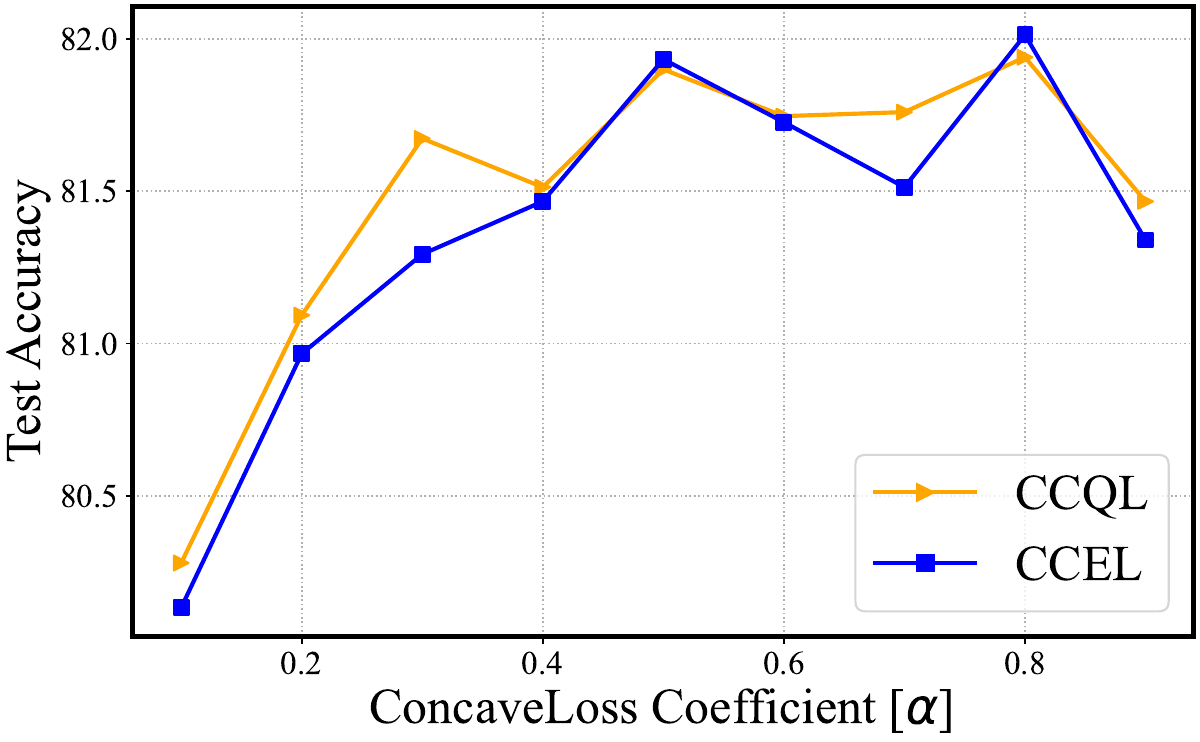}}
        \caption{Utility}
        \label{fig:abltion_utility}
    \end{subfigure}%
    \begin{subfigure}[b]{0.30\textwidth} 
        \centering
        \includegraphics[width=\textwidth]{\detokenize{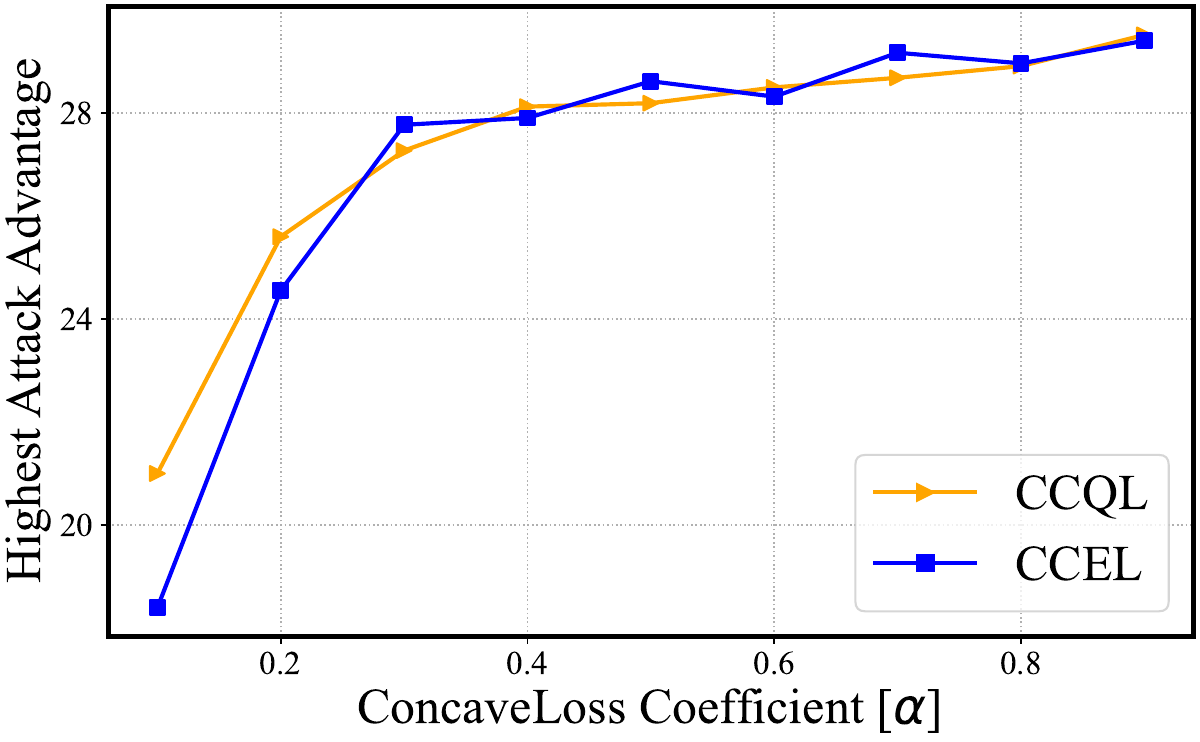}}
        \caption{Privacy}
        \label{fig:abltion_privacy}
    \end{subfigure}
    \begin{subfigure}[b]{0.30\textwidth} 
        \centering
        \includegraphics[width=\textwidth]{\detokenize{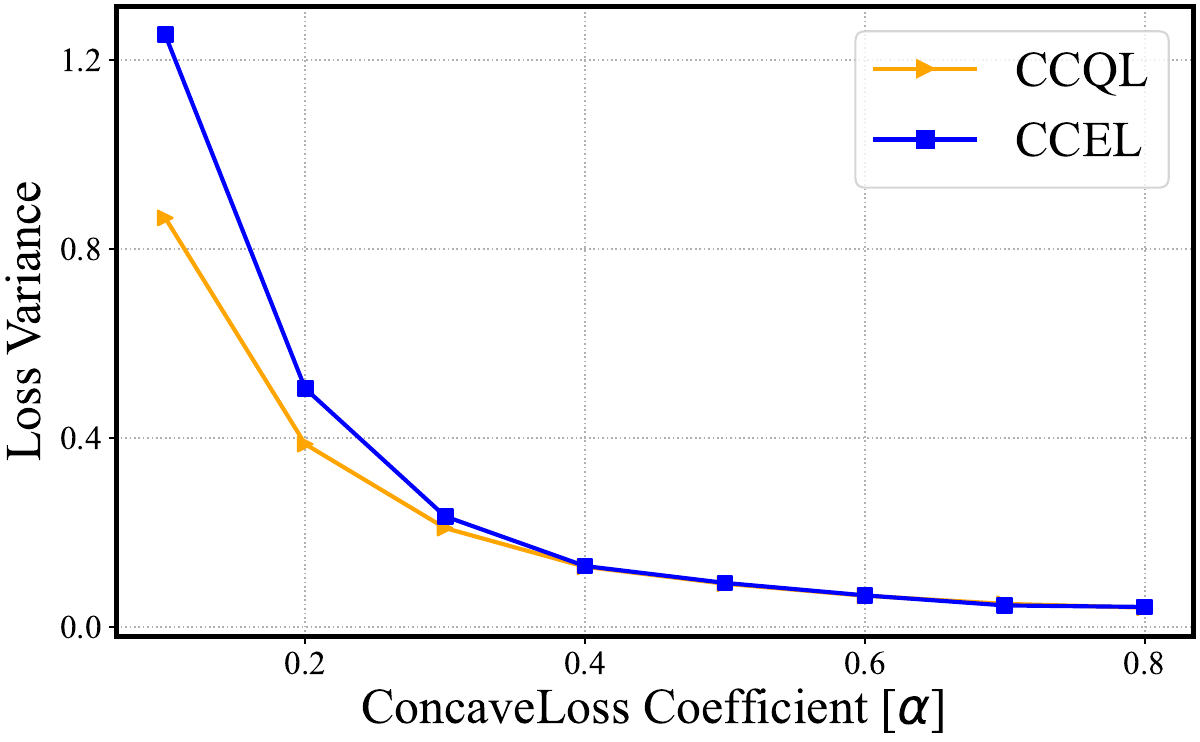}}
        \caption{Variance}
        \label{fig:abltion_variance}
        \end{subfigure}%
    \caption{The effect of \(\alpha\) on utility (test accuracy), privacy (highest attack advantage), and loss variance.}
\end{figure*}
\paragraph{How does $\alpha$ affect utility and privacy?}
In Figure \ref{fig:abltion_utility}, Figure \ref{fig:abltion_privacy} and Figure \ref{fig:abltion_variance}, we conduct an ablation study to examine the impact of the coefficient $\alpha$ in our method on both utility, privacy, and loss variance. The analysis is based on CIFAR-10. As is shown in Figure \ref{fig:abltion_variance}, our findings are in alignment with the insights provided in Theorem \ref{concave loss} and Theorem \ref{convex loss}. As the $\alpha$ decreases, the effect of the concave term becomes more significant, leading to a gradual increase in variance. On the other hand,
a larger $\alpha$ value brings our loss function closer to the cross-entropy loss, thereby increasing the privacy risk. Conversely, a smaller $\alpha$ value leads to a smaller gradient effect, culminating in underfitting, which consequently diminishes accuracy. 
\begin{figure*}[!ht] 
    \centering
    \begin{subfigure}[b]{0.4\textwidth} 
        \centering
        \includegraphics[width=\textwidth]{\detokenize{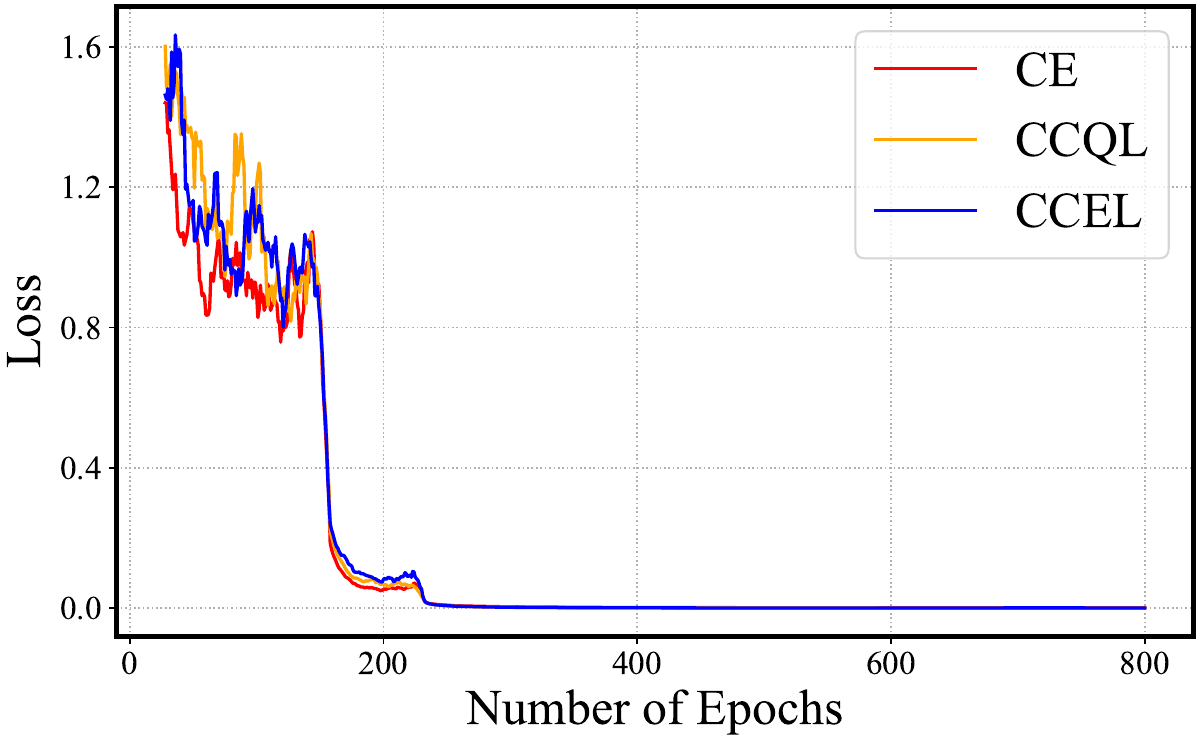}}
        \caption{Loss}
        \label{cifar10_loss_convergence}
    \end{subfigure}
    \hspace{2cm}%
    \begin{subfigure}[b]{0.4\textwidth} 
        \centering
        \includegraphics[width=\textwidth]{\detokenize{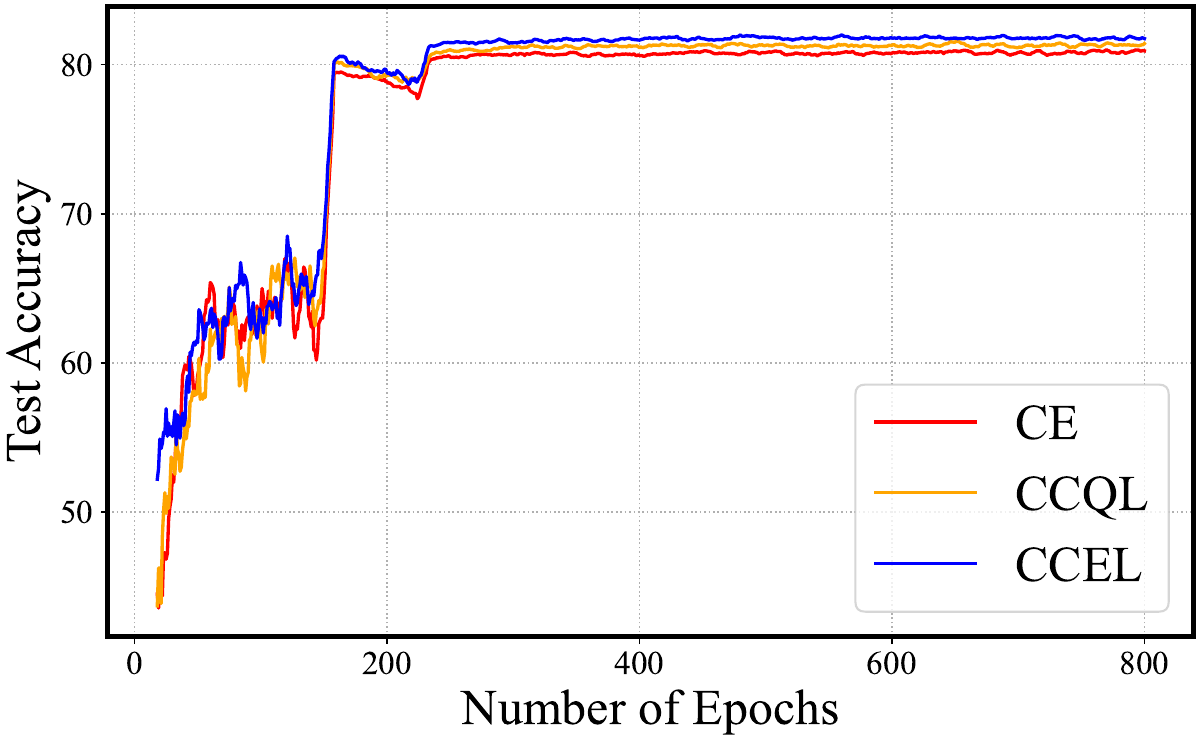}}
        \caption{Test Accuracy}
        \label{cifar10_accuracy_convergence}
    \end{subfigure}
    \caption{Convergence analysis of CCL on CIFAR-10 dataset}
    \label{fig:convergece}
\end{figure*}

\paragraph{Convergence analysis.}
As demonstrated in Equation \ref{gradient analysis concave}, our CCL induces a larger gradient compared to CE loss when the true label confidence \( p_y \) approaches 1. Conversely, during the initial epochs when \( p_y \) is smaller, our approach tends to result in smaller gradient steps.
Although a detailed gradient analysis is presented in Section \ref{Gradient analysis}, it may also ask: Can a model utilizing CCL achieve a stable state? Furthermore, does the implementation of CCL lead to a slower convergence rate? There, we conduct an experiment on CIFAR-10 datasets with fixed $\alpha =0.5$ and plot the training loss curve and the test accuracy curve. As Figure \ref{fig:convergece} shows, both of our proposed CCL functions can converge properly, and there is no significantly longer training time than that of CE to reach convergence. 
\section{Discussion} 
\label{sec:discussion}
\paragraph{How does loss variance affect the attack advantage?} 
Normally, the target model is trained on its training members with the objective of minimizing the error in its predictions. Consequently, the prediction error for a sample within the training dataset would be less than that for a sample within the testing dataset. In this way, the attack model $\mathcal{A}_{\mathrm{loss}}$ \cite{yeom2018privacy} is defined as:
\begin{equation}
    \label{eq:loss_attack_model}
    \mathcal{A}_{\mathrm{loss}} = \mathbb{I}(\mathcal{L}_{\mathrm{CE}}(h_S(\boldsymbol{x}),y) \leqslant  \tau)
\end{equation}
That is, given a sample, we calculate its loss by the target model and then infer it as a member if its loss is smaller than a preset threshold $\tau$\footnote{$\tau$ is determined by shadow models}.

Under the common Gaussian assumption of loss distribution \cite{yeom2018privacy, chen2022relaxloss}, we first introduce a theorem to show how loss variance affects the attack advantage of attack model $\mathcal{A}_{\mathrm{loss}}$ (defined in Equation \ref{eq:loss_attack_model}).
\begin{theorem}  \label{metric adv}
Suppose $\epsilon$ is a random variable denoting loss, such that $\epsilon \sim N(\mu_S,\sigma_{S}^2)$ when $m = 1$ and $\epsilon \sim N(\mu_D,\sigma_{D}^2)$ when $m=0$. Then the membership advantage of $\mathcal{A}_{\mathrm{loss}}$ is:
    \begin{align}
        \mathit{Adv} =& \operatorname{Pr}(\mathcal{A} =1 | m=1) -  \operatorname{Pr}(\mathcal{A} =1 | m=0) \\
         =& \operatorname{Pr}(\epsilon \leqslant \tau | m=1) -  \operatorname{Pr}( \epsilon \leqslant \tau | m=0) \\
         =& \Phi(\frac{\tau - \mu_S}{\sigma_S}) - \Phi(\frac{\tau-\mu_D}{\sigma_D})
    \end{align}
where $\Phi(\cdot)$ is the cumulative distribution function of standard normal distribution.
\label{adv quation}
\end{theorem}
Note that $\operatorname{Pr}(\mathcal{A} =1 | m=0)$ is false positive rates of the adversary, which is expected to be controlled at a small value \cite{leemann2023gaussian,tan2022parameters}. Assume $\tau$ is chosen such that $\Phi(\frac{\tau-\mu_D}{\sigma_D}) = \alpha$, 
then we have:
\begin{equation} 
    \label{adv with sigma}
    \mathit{Adv} = \Phi\{\frac{\Phi^{-1}(\alpha)\sigma_D+\mu_D-\mu_S}{\sigma_S} \} - \alpha
\end{equation}
This implies that increasing the variance of training loss distribution $\sigma_S$ can help to decrease the attack advantage.

\paragraph{Can our method converge?}

In general, non-convex losses do not prevent the convergence of optimization, which has been well-studied in the literature \cite{khaled2020better, zou2019sufficient, garrigos2023handbook, DefossezBBU22}. For example, the work \cite{khaled2020better} proves the convergence of SGD in the non-convex setting under expected smoothness assumption and yields the optimal convergence rate $\mathcal{O}\left(\varepsilon^{-4}\right)$ for finding a stationary point of non-convex smooth functions. Another work \cite{zou2019sufficient} introduces a sufficient condition to guarantee the global convergence of generic Adam/RMSProp optimizers in the non-convex setting.

Here we conclude two common sufficient conditions related to loss function as follows:

\begin{enumerate}[label=(\textbf{C\arabic*})]
    \item The minimum value of $\mathcal{L}$ is lower-bounded. \label{convergence_1}
    \item $\mathcal{L}$ is smooth, i.e., $\Vert\nabla \mathcal{L}(\boldsymbol{x}) - \nabla \mathcal{L}(\boldsymbol{y}) \Vert < L_1 \Vert\boldsymbol{x}-\boldsymbol{y} \Vert$, for all $\boldsymbol{x}$, $\boldsymbol{y} \in \mathbb{R}^{d}$ and for any $L_1 >0$. \label{convergence_2}
\end{enumerate}

We proceed by proving that CCL satisfies the two conditions, supporting the convergence of CCL with popular optimizers, such as SGD, Adam, and RMSProp.

\textbf{CCL satisfies the condition \ref{convergence_1}.} As $\mathcal{L}$ is just linear combination of $\ell_{\mathrm{ccl}}$, so we just prove $\ell_{\mathrm{ccl}}$ is lower bounded. By definition, $\ell_{\mathrm{ccl}}$ is a monotonic decreasing function and closed on the right side of the domain $(0,1]$, so it is lower bounded.

\textbf{CCL satisfies the condition \ref{convergence_2}.} Before we prove C2, we first introduce two lemmas.

\begin{lemma} \label{R1}
Let $A(x): D \subseteq \mathbb{R}^n \rightarrow \mathbb{R}^{m \times p}$ and $B(x): D \subseteq \mathbb{R}^n \rightarrow \mathbb{R}^{p \times q}$ be two Lipschitz continuous matrix functions defined on the same domain $D$, with Lipschitz constants $L_A$ and $L_B$, respectively. If $\Vert A(x)\Vert$ and $\Vert B(x)\Vert$ are bounded for all $x \in D$ by some constants $M_A$ and $M_B$, then the product function $C(x) = A(x)B(x)$ is also Lipschitz continuous on $D$.
\end{lemma}

\begin{lemma} \label{R2}
If two functions $f: \mathbb{R}^n \rightarrow \mathbb{R}^m$ and $g: \mathbb{R}^m \rightarrow \mathbb{R}^p$ are Lipschitz continuous, with Lipschitz constants $L_f$ and $L_g$ respectively, then their composite function $h = g \circ f$ defined as $h(x) = g(f(x))$ is also Lipschitz continuous.
\end{lemma}

The proofs of these two Lemmas \ref{R1} and \ref{R2} are provided in the Appendix \ref{lipschitz}.
By chain rule, we have $\nabla \mathcal{L} = (\nabla \boldsymbol{z})^{\top} \frac{\partial \mathcal{L}}{\partial \boldsymbol{z}}$. Our method differs from CE loss in $\frac{\partial \mathcal{L}}{\partial \boldsymbol{z}}$, we only need to prove that our loss function $\ell_{\mathrm{ccl}}$ is smooth with respect to logits $\boldsymbol{z}$. Note that the confidence in true label $p_y$ can be written as $\boldsymbol{y}^{\top}\boldsymbol{p}$, so we have $\frac{\partial \ell_{\mathrm{ccl}}}{\partial \boldsymbol{z}} = [\alpha- (1-\alpha)\tilde{\ell}^{\prime}(\boldsymbol{y}^{\top}\boldsymbol{p}) \boldsymbol{y}^{\top}\boldsymbol{p} ] (\boldsymbol{p} - \boldsymbol{y})$, where $\boldsymbol{p}$ is probability vector and $\boldsymbol{y}$ is one-hot label vector. Given Lemma \ref{R1} and \ref{R2}, we can simplify the proof to establishing that both $\tilde{\ell}^{\prime}$ and $\boldsymbol{y}^{\top}\boldsymbol{p}$ are Lipschitz continuous.

Since $\tilde{\ell} \in C^{2}[0, 1]$, there exits an upper bound $B$ such that $\tilde{\ell}^{\prime \prime} \leqslant B$, which follows that $\Vert \tilde{\ell}^{\prime}(x) - \tilde{\ell}^{\prime} (y) \Vert \leqslant B \Vert  x-y\Vert$. With the fact that Softmax function is Lipschitz continuous [5], we have $\Vert \boldsymbol{y}^{\top}\boldsymbol{p}_1 - \boldsymbol{y}^{\top}\boldsymbol{p}_2 \Vert \leqslant \Vert \boldsymbol{y}  \Vert \cdot \Vert \boldsymbol{p}_1 - \boldsymbol{p}_2 \Vert =  \Vert\boldsymbol{p}_1 - \boldsymbol{p}_2 \Vert \leqslant  L_2 \Vert \boldsymbol{z_1} - \boldsymbol{z_1} \Vert$.

By Lemma \ref{R1} and Lemma \ref{R2}, we conclude that $\frac{\partial \ell_{\mathrm{ccl}}}{\partial \boldsymbol{z}}$ is Lipschitz continuous, followed by it is smoothness.

\section{Related Work}
\paragraph{Membership Inference Attacks (MIAs).}
Membership Inference Attacks (MIAs), first introduced by \citet{shokri2017membership} for machine learning, utilize multiple shadow models and a neural network-based attack model to identify a target model’s predictions on member versus non-member data, that is  NN-based attack\cite{zhang2023mida,nasr2019comprehensive}. Metric-based attack computes custom metrics such as Loss \cite{yeom2018privacy}, Confidence\cite{liu2019socinf}, Entropy \cite{salem2018ml}, Modified-Entropy \cite{song2021systematic}, and gradient norm \cite{leemann2023gaussian, nasr2019comprehensive, sablayrolles2019white} to derive a threshold for distinction. As an extension to metric-based attacks, recent works \cite{lopez2023scalable, carlini2022membership, ye2022enhanced, watson2021importance} design per-example attacks.
Boundary-based attacks \cite{li2021membership, choquette2021label} gauge the necessary perturbation magnitude for membership inference. Augmentation-based attacks \cite{choquette2021label,ko2023practical} leverage the resilience of training samples to data augmentation compared to testing samples. 

Beyond supervised classification, MIAs have been extended to additional fields, such as graph embedding models \cite{wang2023link, duddu2020quantifying}, graph neural network\cite{liu2022membership,conti2022label, olatunji2021membership}, generative models \cite{pang2023white, dubinski2024towards, van2023membership, hu2021membership, chen2020gan, liu2019performing}, contrastive learning \cite{liu2021encodermi, ko2023practical}, language models \cite{mattern2023membership,mireshghallah2022quantifying}.
\paragraph{Defend against MIAs.}
The existing defense technologies can be classified into four categories\cite{hu2023defenses}: regularization, transfer learning, information perturbation, and generative models-based. 

(\romannumeral1) Regularization: Many papers \cite{leino2020stolen,salem2018ml, yeom2018privacy, shokri2017membership} have pointed out that overfitting is a major factor in the success of MIAs, hence regularization
technology such as Label Smoothing \cite{guo2017calibration}, Confidence Penalty \cite{pereyra2017regularizing}, and Early Stopping \cite{yao2007early}, Adversary regularization \cite{nasr2018machine}, Dropout~\cite{srivastava2014dropout}, Pruning\cite{wang2020against}, HAMP \cite{chen2023overconfidence}, and RelaxLoss \cite{chen2022relaxloss} can certainly defend against MIAs. 

(\romannumeral2) Transfer Learning has been shown to effectively protect member privacy by using knowledge from similar but different data, reducing direct access to sensitive target data. In particular, Knowledge distillation \cite{mazzone2022repeated, hinton2015distilling, shejwalkar2021membership, zheng2021resisting, DBLP:conf/uss/TangMSSNHM22} uses a large teacher model to train a smaller student model, transferring the knowledge while retaining similar accuracy. Domain adaptation \cite{weiss2016survey, huang2021defense, huang2021damia} transfers knowledge from a source domain to a related but different target domain by extracting shared representations. 

(\romannumeral3) Information perturbation protects privacy information by adding customized noise to the data, training procession, and outputs. Specially, this methodology is typically classified into three methods: differential privacy, output perturbation, and data augmentation. Differential privacy \cite{tan2022parameters, dwork2008differential,chen2020differential,jayaraman2019evaluating,nasr2021adversary,rahman2018membership, kim2021federated, truex2019effects}  adds noise perturbations to the real data, ensuring that the results of an algorithm on adjacent datasets (differing by one element) are statistically indistinguishable.  Output perturbation \cite{xue2022use, jia2019memguard} based on the intuition that black-box MIAs can only utilize this output information to make inferences, so lightly altering the results returned by the model weakens the performance of MIAs. Data perturbation \cite{chen2021gan, wang2019miasec, kandpal2022deduplicating} add perturbations directly to the data, making it more challenging for attackers to infer whether specific data points were used in training the model. Techniques such as data augmentation, including rotation, clipping, and mix-up \cite{li2021mixup, kaya2021does}, can also fall into this category.

(\romannumeral4) Generative models-based methods \cite{chen2021protect,hu2022defending, paul2021defending} generated substitute training data using generative models to reduce information leakage. Our approach is a regularization technique integrated into the loss function.

\section{Conclusion}
In this paper, we introduce Convex-Concave Loss, a simple method for formulating loss functions that can help defend against MIAs. Specifically, we propose to integrate a concave term with CE loss to magnify loss variance. As a result, our proposed method could mitigate privacy risks by reducing the gap between training and testing loss distribution. Moreover, we present theoretical analyses of how convex and concave loss functions affect loss variance during optimization, which is the key insight enlightening and certifying our method. Extensive experiments show that CCL can improve the privacy-utility trade-off. 

A few open questions remain: Firstly, our method aims to increase the variance of output metrics, but its role in defending against label-only attacks (such as data augmentation attacks) remains unexplored. Moreover, our method cannot break the trade-off between utility and MIA defense, which might be a potential direction for future work. 

\section*{Acknowledgements}

This research is supported by the Shenzhen Fundamental Research Program (Grant No.JCYJ20230807091809020). Huiping Zhuang is supported by the National Natural Science Foundation of China (Grant No. 6230070401), Xiaofeng Cao is supported by the National Natural Science Foundation of China (Grant No.62206108). We gratefully acknowledge the support of the Center for Computational Science and Engineering at the Southern University of Science and Technology.

\section*{Impact Statement}
This paper presents work whose goal is to advance the field of Machine Learning. There are many potential societal consequences of our work, none of which we feel must be specifically highlighted here.

\bibliography{example_paper}

\begin{thebibliography}{86}
\providecommand{\natexlab}[1]{#1}
\providecommand{\url}[1]{\texttt{#1}}
\expandafter\ifx\csname urlstyle\endcsname\relax
  \providecommand{\doi}[1]{doi: #1}\else
  \providecommand{\doi}{doi: \begingroup \urlstyle{rm}\Url}\fi

\bibitem[Arshad et~al.(2021)Arshad, Arshad, Khan, and Parkinson]{arshad2021analysis}
Arshad, S., Arshad, J., Khan, M.~M., and Parkinson, S.
\newblock Analysis of security and privacy challenges for dna-genomics applications and databases.
\newblock \emph{Journal of Biomedical Informatics}, 119:\penalty0 103815, 2021.

\bibitem[Carlini et~al.(2022)Carlini, Chien, Nasr, Song, Terzis, and Tramer]{carlini2022membership}
Carlini, N., Chien, S., Nasr, M., Song, S., Terzis, A., and Tramer, F.
\newblock Membership inference attacks from first principles.
\newblock In \emph{2022 IEEE Symposium on Security and Privacy (SP)}, pp.\  1897--1914. IEEE, 2022.

\bibitem[Chen et~al.(2020{\natexlab{a}})Chen, Yu, Zhang, and Fritz]{chen2020gan}
Chen, D., Yu, N., Zhang, Y., and Fritz, M.
\newblock Gan-leaks: A taxonomy of membership inference attacks against generative models.
\newblock In \emph{Proceedings of the 2020 ACM SIGSAC conference on computer and communications security}, pp.\  343--362, 2020{\natexlab{a}}.

\bibitem[Chen et~al.(2022)Chen, Yu, and Fritz]{chen2022relaxloss}
Chen, D., Yu, N., and Fritz, M.
\newblock Relaxloss: Defending membership inference attacks without losing utility.
\newblock In \emph{International Conference on Learning Representations}, 2022.

\bibitem[Chen et~al.(2020{\natexlab{b}})Chen, Wang, and Shi]{chen2020differential}
Chen, J., Wang, W.~H., and Shi, X.
\newblock Differential privacy protection against membership inference attack on machine learning for genomic data.
\newblock In \emph{BIOCOMPUTING 2021: Proceedings of the Pacific Symposium}, pp.\  26--37. World Scientific, 2020{\natexlab{b}}.

\bibitem[Chen et~al.(2021{\natexlab{a}})Chen, Guo, Zheng, and Chen]{chen2021protect}
Chen, J., Guo, Y., Zheng, Q., and Chen, H.
\newblock Protect privacy of deep classification networks by exploiting their generative power.
\newblock \emph{Machine Learning}, 110:\penalty0 651--674, 2021{\natexlab{a}}.

\bibitem[Chen et~al.(2021{\natexlab{b}})Chen, Wang, Gao, and Shi]{chen2021gan}
Chen, J., Wang, W.~H., Gao, H., and Shi, X.
\newblock Par-gan: improving the generalization of generative adversarial networks against membership inference attacks.
\newblock In \emph{Proceedings of the 27th ACM SIGKDD Conference on Knowledge Discovery \& Data mining}, pp.\  127--137, 2021{\natexlab{b}}.

\bibitem[Chen \& Pattabiraman(2024)Chen and Pattabiraman]{chen2023overconfidence}
Chen, Z. and Pattabiraman, K.
\newblock Overconfidence is a dangerous thing: Mitigating membership inference attacks by enforcing less confident prediction.
\newblock In \emph{30th Annual Network and Distributed System Security Symposium}, 2024.

\bibitem[Choquette-Choo et~al.(2021)Choquette-Choo, Tramer, Carlini, and Papernot]{choquette2021label}
Choquette-Choo, C.~A., Tramer, F., Carlini, N., and Papernot, N.
\newblock Label-only membership inference attacks.
\newblock In \emph{Proceedings of the 38th International Conference on Machine Learning}, pp.\  1964--1974. PMLR, 2021.

\bibitem[Conti et~al.(2022)Conti, Li, Picek, and Xu]{conti2022label}
Conti, M., Li, J., Picek, S., and Xu, J.
\newblock Label-only membership inference attack against node-level graph neural networks.
\newblock In \emph{Proceedings of the 15th ACM Workshop on Artificial Intelligence and Security}, pp.\  1--12, 2022.

\bibitem[D{\'{e}}fossez et~al.(2022)D{\'{e}}fossez, Bottou, Bach, and Usunier]{DefossezBBU22}
D{\'{e}}fossez, A., Bottou, L., Bach, F.~R., and Usunier, N.
\newblock A simple convergence proof of adam and adagrad.
\newblock \emph{Trans. Mach. Learn. Res.}, 2022, 2022.
\newblock URL \url{https://openreview.net/forum?id=ZPQhzTSWA7}.

\bibitem[Dubi{\'n}ski et~al.(2024)Dubi{\'n}ski, Kowalczuk, Pawlak, Rokita, Trzci{\'n}ski, and Morawiecki]{dubinski2024towards}
Dubi{\'n}ski, J., Kowalczuk, A., Pawlak, S., Rokita, P., Trzci{\'n}ski, T., and Morawiecki, P.
\newblock Towards more realistic membership inference attacks on large diffusion models.
\newblock In \emph{Proceedings of the IEEE/CVF Winter Conference on Applications of Computer Vision}, pp.\  4860--4869, 2024.

\bibitem[Duddu et~al.(2020)Duddu, Boutet, and Shejwalkar]{duddu2020quantifying}
Duddu, V., Boutet, A., and Shejwalkar, V.
\newblock Quantifying privacy leakage in graph embedding.
\newblock In \emph{MobiQuitous 2020-17th EAI International Conference on Mobile and Ubiquitous Systems: Computing, Networking and Services}, pp.\  76--85, 2020.

\bibitem[Dwork(2008)]{dwork2008differential}
Dwork, C.
\newblock Differential privacy: A survey of results.
\newblock In \emph{International Conference on Theory and Applications of Models of Computation}, pp.\  1--19. Springer, 2008.

\bibitem[Garrigos \& Gower(2023)Garrigos and Gower]{garrigos2023handbook}
Garrigos, G. and Gower, R.~M.
\newblock Handbook of convergence theorems for (stochastic) gradient methods.
\newblock \emph{arXiv preprint arXiv:2301.11235}, 2023.

\bibitem[Guo et~al.(2017)Guo, Pleiss, Sun, and Weinberger]{guo2017calibration}
Guo, C., Pleiss, G., Sun, Y., and Weinberger, K.~Q.
\newblock On calibration of modern neural networks.
\newblock In \emph{International Conference on Machine Learning}, pp.\  1321--1330. PMLR, 2017.

\bibitem[He et~al.(2016)He, Zhang, Ren, and Sun]{he2016deep}
He, K., Zhang, X., Ren, S., and Sun, J.
\newblock Deep residual learning for image recognition.
\newblock In \emph{Proceedings of the IEEE conference on Computer Vision and Pattern Recognition}, pp.\  770--778, 2016.

\bibitem[Hinton et~al.(2015)Hinton, Vinyals, and Dean]{hinton2015distilling}
Hinton, G., Vinyals, O., and Dean, J.
\newblock Distilling the knowledge in a neural network.
\newblock \emph{arXiv preprint arXiv:1503.02531}, 2015.

\bibitem[Hu \& Pang(2021)Hu and Pang]{hu2021membership}
Hu, H. and Pang, J.
\newblock Membership inference attacks against gans by leveraging over-representation regions.
\newblock In \emph{Proceedings of the 2021 ACM SIGSAC Conference on Computer and Communications Security}, pp.\  2387--2389, 2021.

\bibitem[Hu et~al.(2022{\natexlab{a}})Hu, Salcic, Sun, Dobbie, Yu, and Zhang]{hu2022membership}
Hu, H., Salcic, Z., Sun, L., Dobbie, G., Yu, P.~S., and Zhang, X.
\newblock Membership inference attacks on machine learning: A survey.
\newblock \emph{ACM Computing Surveys (CSUR)}, 54\penalty0 (11s):\penalty0 1--37, 2022{\natexlab{a}}.

\bibitem[Hu et~al.(2022{\natexlab{b}})Hu, Li, Lin, Peng, Zhang, Zhang, and Dong]{hu2022defending}
Hu, L., Li, J., Lin, G., Peng, S., Zhang, Z., Zhang, Y., and Dong, C.
\newblock Defending against membership inference attacks with high utility by gan.
\newblock \emph{IEEE Transactions on Dependable and Secure Computing}, 2022{\natexlab{b}}.

\bibitem[Hu et~al.(2023)Hu, Yan, Yan, Li, Huang, Zhang, Dong, and Yang]{hu2023defenses}
Hu, L., Yan, A., Yan, H., Li, J., Huang, T., Zhang, Y., Dong, C., and Yang, C.
\newblock Defenses to membership inference attacks: A survey.
\newblock \emph{ACM Computing Surveys}, 56\penalty0 (4):\penalty0 1--34, 2023.

\bibitem[Huang et~al.(2017)Huang, Liu, Van Der~Maaten, and Weinberger]{huang2017densely}
Huang, G., Liu, Z., Van Der~Maaten, L., and Weinberger, K.~Q.
\newblock Densely connected convolutional networks.
\newblock In \emph{Proceedings of the IEEE conference on Computer Vision and Pattern Recognition}, pp.\  4700--4708, 2017.

\bibitem[Huang(2021)]{huang2021defense}
Huang, H.
\newblock Defense against membership inference attack applying domain adaptation with addictive noise.
\newblock \emph{Journal of Computer and Communications}, 9\penalty0 (05):\penalty0 92--108, 2021.

\bibitem[Huang et~al.(2021)Huang, Luo, Zeng, Weng, Zhang, and Yang]{huang2021damia}
Huang, H., Luo, W., Zeng, G., Weng, J., Zhang, Y., and Yang, A.
\newblock Damia: leveraging domain adaptation as a defense against membership inference attacks.
\newblock \emph{IEEE Transactions on Dependable and Secure Computing}, 19\penalty0 (5):\penalty0 3183--3199, 2021.

\bibitem[Irolla \& Châtel(2019)Irolla and Châtel]{truex2018towards}
Irolla, P. and Châtel, G.
\newblock Demystifying the membership inference attack.
\newblock In \emph{2019 12th CMI Conference on Cybersecurity and Privacy (CMI)}, pp.\  1--7, 2019.
\newblock \doi{10.1109/CMI48017.2019.8962136}.

\bibitem[Jayaraman \& Evans(2019)Jayaraman and Evans]{jayaraman2019evaluating}
Jayaraman, B. and Evans, D.
\newblock Evaluating differentially private machine learning in practice.
\newblock In \emph{28th USENIX Security Symposium (USENIX Security 19)}, pp.\  1895--1912, 2019.

\bibitem[Jia et~al.(2019)Jia, Salem, Backes, Zhang, and Gong]{jia2019memguard}
Jia, J., Salem, A., Backes, M., Zhang, Y., and Gong, N.~Z.
\newblock Memguard: Defending against black-box membership inference attacks via adversarial examples.
\newblock In \emph{Proceedings of the 2019 ACM SIGSAC Conference on Computer and Communications Security}, pp.\  259--274, 2019.

\bibitem[Kaggle(2014)]{KaggleValuedShoppers2014}
Kaggle.
\newblock Acquire valued shoppers challenge, 2014.
\newblock URL \url{https://www.kaggle.com/c/acquire-valued-shoppers-challenge/data}.

\bibitem[Kandpal et~al.(2022)Kandpal, Wallace, and Raffel]{kandpal2022deduplicating}
Kandpal, N., Wallace, E., and Raffel, C.
\newblock Deduplicating training data mitigates privacy risks in language models.
\newblock In \emph{International Conference on Machine Learning}, pp.\  10697--10707. PMLR, 2022.

\bibitem[Kaya \& Dumitras(2021)Kaya and Dumitras]{kaya2021does}
Kaya, Y. and Dumitras, T.
\newblock When does data augmentation help with membership inference attacks?
\newblock In \emph{International Conference on Machine Learning}, pp.\  5345--5355. PMLR, 2021.

\bibitem[Khaled \& Richt{\'a}rik(2020)Khaled and Richt{\'a}rik]{khaled2020better}
Khaled, A. and Richt{\'a}rik, P.
\newblock Better theory for sgd in the nonconvex world.
\newblock \emph{arXiv preprint arXiv:2002.03329}, 2020.

\bibitem[Kim et~al.(2021)Kim, G{\"u}nl{\"u}, and Schaefer]{kim2021federated}
Kim, M., G{\"u}nl{\"u}, O., and Schaefer, R.~F.
\newblock Federated learning with local differential privacy: Trade-offs between privacy, utility, and communication.
\newblock In \emph{ICASSP 2021-2021 IEEE International Conference on Acoustics, Speech and Signal Processing (ICASSP)}, pp.\  2650--2654. IEEE, 2021.

\bibitem[Ko et~al.(2023)Ko, Jin, Wang, and Jia]{ko2023practical}
Ko, M., Jin, M., Wang, C., and Jia, R.
\newblock Practical membership inference attacks against large-scale multi-modal models: A pilot study.
\newblock In \emph{Proceedings of the IEEE/CVF International Conference on Computer Vision}, pp.\  4871--4881, 2023.

\bibitem[Krizhevsky et~al.(2009)Krizhevsky, Hinton, et~al.]{krizhevsky2009learning}
Krizhevsky, A., Hinton, G., et~al.
\newblock Learning multiple layers of features from tiny images.
\newblock \emph{Master's thesis, Department of Computer Science, University of Toronto}, 2009.

\bibitem[Leemann et~al.(2023)Leemann, Pawelczyk, and Kasneci]{leemann2023gaussian}
Leemann, T., Pawelczyk, M., and Kasneci, G.
\newblock Gaussian membership inference privacy.
\newblock In \emph{37th Conference on Neural Information Processing Systems (NeurIPS)}, 2023.

\bibitem[Leino \& Fredrikson(2020)Leino and Fredrikson]{leino2020stolen}
Leino, K. and Fredrikson, M.
\newblock Stolen memories: Leveraging model memorization for calibrated $\{$White-Box$\}$ membership inference.
\newblock In \emph{29th USENIX Security Symposium (USENIX Security 20)}, pp.\  1605--1622, 2020.

\bibitem[Li et~al.(2021)Li, Li, and Ribeiro]{li2021mixup}
Li, J., Li, N., and Ribeiro, B.
\newblock Membership inference attacks and defenses in classification models.
\newblock In \emph{Proceedings of the Eleventh ACM Conference on Data and Application Security and Privacy}, pp.\  5--16, 2021.

\bibitem[Li \& Zhang(2021)Li and Zhang]{li2021membership}
Li, Z. and Zhang, Y.
\newblock Membership leakage in label-only exposures.
\newblock In \emph{Proceedings of the 2021 ACM SIGSAC Conference on Computer and Communications Security}, pp.\  880--895, 2021.

\bibitem[Lin et~al.(2017)Lin, Goyal, Girshick, He, and Doll{\'a}r]{lin2017focal}
Lin, T.-Y., Goyal, P., Girshick, R., He, K., and Doll{\'a}r, P.
\newblock Focal loss for dense object detection.
\newblock In \emph{Proceedings of the IEEE International Conference on Computer Vision}, pp.\  2980--2988, 2017.

\bibitem[Liu et~al.(2019{\natexlab{a}})Liu, Wang, Peng, Huang, Li, and Cheng]{liu2019socinf}
Liu, G., Wang, C., Peng, K., Huang, H., Li, Y., and Cheng, W.
\newblock Socinf: Membership inference attacks on social media health data with machine learning.
\newblock \emph{IEEE Transactions on Computational Social Systems}, 6\penalty0 (5):\penalty0 907--921, 2019{\natexlab{a}}.

\bibitem[Liu et~al.(2021)Liu, Jia, Qu, and Gong]{liu2021encodermi}
Liu, H., Jia, J., Qu, W., and Gong, N.~Z.
\newblock Encodermi: Membership inference against pre-trained encoders in contrastive learning.
\newblock In \emph{Proceedings of the 2021 ACM SIGSAC Conference on Computer and Communications Security}, pp.\  2081--2095, 2021.

\bibitem[Liu et~al.(2019{\natexlab{b}})Liu, Xiao, Li, and Gao]{liu2019performing}
Liu, K.~S., Xiao, C., Li, B., and Gao, J.
\newblock Performing co-membership attacks against deep generative models.
\newblock In \emph{2019 IEEE International Conference on Data Mining (ICDM)}, pp.\  459--467. IEEE, 2019{\natexlab{b}}.

\bibitem[Liu et~al.(2022)Liu, Zhang, Chen, Lin, and Li]{liu2022membership}
Liu, Z., Zhang, X., Chen, C., Lin, S., and Li, J.
\newblock Membership inference attacks against robust graph neural network.
\newblock In \emph{International Symposium on Cyberspace Safety and Security}, pp.\  259--273. Springer, 2022.

\bibitem[Lopez et~al.(2023)Lopez, Tang, Kearns, Morgenstern, Roth, and Wu]{lopez2023scalable}
Lopez, M.~B., Tang, S., Kearns, M., Morgenstern, J., Roth, A., and Wu, Z.~S.
\newblock Scalable membership inference attacks via quantile regression.
\newblock In \emph{Thirty-seventh Conference on Neural Information Processing Systems}, 2023.

\bibitem[Mahalle et~al.(2018)Mahalle, Yong, Tao, and Shen]{mahalle2018data}
Mahalle, A., Yong, J., Tao, X., and Shen, J.
\newblock Data privacy and system security for banking and financial services industry based on cloud computing infrastructure.
\newblock In \emph{2018 IEEE 22nd International Conference on Computer Supported Cooperative Work in Design ((CSCWD))}, pp.\  407--413. IEEE, 2018.

\bibitem[Malinin \& Gales(2018)Malinin and Gales]{malinin2018predictive}
Malinin, A. and Gales, M.
\newblock Predictive uncertainty estimation via prior networks.
\newblock \emph{Advances in Neural Information Processing Systems}, 31, 2018.

\bibitem[Mattern et~al.(2023)Mattern, Mireshghallah, Jin, Sch{\"{o}}lkopf, Sachan, and Berg{-}Kirkpatrick]{mattern2023membership}
Mattern, J., Mireshghallah, F., Jin, Z., Sch{\"{o}}lkopf, B., Sachan, M., and Berg{-}Kirkpatrick, T.
\newblock Membership inference attacks against language models via neighbourhood comparison.
\newblock In \emph{{ACL} (Findings)}, pp.\  11330--11343. Association for Computational Linguistics, 2023.

\bibitem[Mazzone et~al.(2022)Mazzone, van~den Heuvel, Huber, Verdecchia, Everts, Hahn, and Peter]{mazzone2022repeated}
Mazzone, F., van~den Heuvel, L., Huber, M., Verdecchia, C., Everts, M., Hahn, F., and Peter, A.
\newblock Repeated knowledge distillation with confidence masking to mitigate membership inference attacks.
\newblock In \emph{Proceedings of the 15th ACM Workshop on Artificial Intelligence and Security}, pp.\  13--24, 2022.

\bibitem[Mireshghallah et~al.(2022)Mireshghallah, Goyal, Uniyal, Berg{-}Kirkpatrick, and Shokri]{mireshghallah2022quantifying}
Mireshghallah, F., Goyal, K., Uniyal, A., Berg{-}Kirkpatrick, T., and Shokri, R.
\newblock Quantifying privacy risks of masked language models using membership inference attacks.
\newblock In \emph{Proceedings of the 2022 Conference on Empirical Methods in Natural Language Processing}, pp.\  8332--8347. Association for Computational Linguistics, 2022.

\bibitem[Nasr et~al.(2018)Nasr, Shokri, and Houmansadr]{nasr2018machine}
Nasr, M., Shokri, R., and Houmansadr, A.
\newblock Machine learning with membership privacy using adversarial regularization.
\newblock In \emph{Proceedings of the 2018 ACM SIGSAC Conference on Computer and Communications Security}, pp.\  634--646, 2018.

\bibitem[Nasr et~al.(2019)Nasr, Shokri, and Houmansadr]{nasr2019comprehensive}
Nasr, M., Shokri, R., and Houmansadr, A.
\newblock Comprehensive privacy analysis of deep learning: Passive and active white-box inference attacks against centralized and federated learning.
\newblock In \emph{2019 IEEE Symposium on Security and Privacy (SP)}, pp.\  739--753. IEEE, 2019.

\bibitem[Nasr et~al.(2021)Nasr, Songi, Thakurta, Papernot, and Carlin]{nasr2021adversary}
Nasr, M., Songi, S., Thakurta, A., Papernot, N., and Carlin, N.
\newblock Adversary instantiation: Lower bounds for differentially private machine learning.
\newblock In \emph{2021 IEEE Symposium on Security and Privacy (SP)}, pp.\  866--882. IEEE, 2021.

\bibitem[Oehlert(1992)]{oehlert1992note}
Oehlert, G.~W.
\newblock A note on the delta method.
\newblock \emph{American Statistician}, pp.\  27--29, 1992.

\bibitem[Olatunji et~al.(2021)Olatunji, Nejdl, and Khosla]{olatunji2021membership}
Olatunji, I.~E., Nejdl, W., and Khosla, M.
\newblock Membership inference attack on graph neural networks.
\newblock In \emph{2021 Third IEEE International Conference on Trust, Privacy and Security in Intelligent Systems and Applications (TPS-ISA)}, pp.\  11--20. IEEE, 2021.

\bibitem[Pang et~al.(2023)Pang, Wang, Kang, Huai, and Zhang]{pang2023white}
Pang, Y., Wang, T., Kang, X., Huai, M., and Zhang, Y.
\newblock White-box membership inference attacks against diffusion models.
\newblock \emph{arXiv preprint arXiv:2308.06405}, 2023.

\bibitem[Paul et~al.(2021)Paul, Cao, Zhang, and Burlina]{paul2021defending}
Paul, W., Cao, Y., Zhang, M., and Burlina, P.
\newblock Defending medical image diagnostics against privacy attacks using generative methods: Application to retinal diagnostics.
\newblock In \emph{Clinical Image-Based Procedures, Distributed and Collaborative Learning, Artificial Intelligence for Combating COVID-19 and Secure and Privacy-Preserving Machine Learning: 10th Workshop, CLIP 2021, Second Workshop, DCL 2021, First Workshop, LL-COVID19 2021, and First Workshop and Tutorial, PPML 2021, Held in Conjunction with MICCAI 2021, Strasbourg, France, September 27 and October 1, 2021, Proceedings 2}, pp.\  174--187. Springer, 2021.

\bibitem[Pereyra et~al.(2017)Pereyra, Tucker, Chorowski, Kaiser, and Hinton]{pereyra2017regularizing}
Pereyra, G., Tucker, G., Chorowski, J., Kaiser, L., and Hinton, G.~E.
\newblock Regularizing neural networks by penalizing confident output distributions.
\newblock In \emph{International Conference on Learning Representations}, 2017.

\bibitem[Rahman et~al.(2018)Rahman, Rahman, Lagani{\`e}re, Mohammed, and Wang]{rahman2018membership}
Rahman, M.~A., Rahman, T., Lagani{\`e}re, R., Mohammed, N., and Wang, Y.
\newblock Membership inference attack against differentially private deep learning model.
\newblock \emph{Trans. Data Priv.}, 11\penalty0 (1):\penalty0 61--79, 2018.

\bibitem[Russakovsky et~al.(2015)Russakovsky, Deng, Su, Krause, Satheesh, Ma, Huang, Karpathy, Khosla, Bernstein, et~al.]{russakovsky2015imagenet}
Russakovsky, O., Deng, J., Su, H., Krause, J., Satheesh, S., Ma, S., Huang, Z., Karpathy, A., Khosla, A., Bernstein, M., et~al.
\newblock Imagenet large scale visual recognition challenge.
\newblock \emph{International Journal of Computer Vision}, 115:\penalty0 211--252, 2015.

\bibitem[Sablayrolles et~al.(2019)Sablayrolles, Douze, Schmid, Ollivier, and J{\'e}gou]{sablayrolles2019white}
Sablayrolles, A., Douze, M., Schmid, C., Ollivier, Y., and J{\'e}gou, H.
\newblock White-box vs black-box: Bayes optimal strategies for membership inference.
\newblock In \emph{International Conference on Machine Learning}, pp.\  5558--5567. PMLR, 2019.

\bibitem[Salem et~al.(2019)Salem, Zhang, Humbert, Berrang, Fritz, and Backes]{salem2018ml}
Salem, A., Zhang, Y., Humbert, M., Berrang, P., Fritz, M., and Backes, M.
\newblock Ml-leaks: Model and data independent membership inference attacks and defenses on machine learning models.
\newblock In \emph{{Network and Distributed System Security (NDSS) Symposium}}. The Internet Society, 2019.

\bibitem[Shejwalkar \& Houmansadr(2021)Shejwalkar and Houmansadr]{shejwalkar2021membership}
Shejwalkar, V. and Houmansadr, A.
\newblock Membership privacy for machine learning models through knowledge transfer.
\newblock In \emph{Proceedings of the AAAI Conference on Artificial Intelligence}, volume~35, pp.\  9549--9557, 2021.

\bibitem[Shokri et~al.(2017)Shokri, Stronati, Song, and Shmatikov]{shokri2017membership}
Shokri, R., Stronati, M., Song, C., and Shmatikov, V.
\newblock Membership inference attacks against machine learning models.
\newblock In \emph{2017 IEEE Symposium on Security and Privacy (SP)}, pp.\  3--18. IEEE, 2017.

\bibitem[Song \& Mittal(2021)Song and Mittal]{song2021systematic}
Song, L. and Mittal, P.
\newblock Systematic evaluation of privacy risks of machine learning models.
\newblock In \emph{30th USENIX Security Symposium (USENIX Security 21)}, pp.\  2615--2632, 2021.

\bibitem[Srivastava et~al.(2014)Srivastava, Hinton, Krizhevsky, Sutskever, and Salakhutdinov]{srivastava2014dropout}
Srivastava, N., Hinton, G., Krizhevsky, A., Sutskever, I., and Salakhutdinov, R.
\newblock Dropout: a simple way to prevent neural networks from overfitting.
\newblock \emph{The Journal of Machine Learning Research}, 15\penalty0 (1):\penalty0 1929--1958, 2014.

\bibitem[Tan et~al.(2022)Tan, Mason, Javadi, and Baraniuk]{tan2022parameters}
Tan, J., Mason, B., Javadi, H., and Baraniuk, R.
\newblock Parameters or privacy: A provable tradeoff between overparameterization and membership inference.
\newblock \emph{Advances in Neural Information Processing Systems}, 35:\penalty0 17488--17500, 2022.

\bibitem[Tang et~al.(2022)Tang, Mahloujifar, Song, Shejwalkar, Nasr, Houmansadr, and Mittal]{DBLP:conf/uss/TangMSSNHM22}
Tang, X., Mahloujifar, S., Song, L., Shejwalkar, V., Nasr, M., Houmansadr, A., and Mittal, P.
\newblock Mitigating membership inference attacks by self-distillation through a novel ensemble architecture.
\newblock In \emph{{USENIX} Security Symposium}, pp.\  1433--1450. {USENIX} Association, 2022.

\bibitem[{Texas Department of State Health Services}(2006)]{TexasInpatient2006}
{Texas Department of State Health Services}.
\newblock Texas hospital inpatient discharge public use data file, 2006.
\newblock URL \url{https://www.dshs.texas.gov/thcic/hospitals/Inpatientpudf.shtm}.

\bibitem[Truex et~al.(2019)Truex, Liu, Gursoy, Wei, and Yu]{truex2019effects}
Truex, S., Liu, L., Gursoy, M.~E., Wei, W., and Yu, L.
\newblock Effects of differential privacy and data skewness on membership inference vulnerability.
\newblock In \emph{2019 First IEEE International Conference on Trust, Privacy and Security in Intelligent Systems and Applications (TPS-ISA)}, pp.\  82--91. IEEE, 2019.

\bibitem[van Breugel et~al.(2023)van Breugel, Sun, Qian, and van~der Schaar]{van2023membership}
van Breugel, B., Sun, H., Qian, Z., and van~der Schaar, M.
\newblock Membership inference attacks against synthetic data through overfitting detection.
\newblock In \emph{{AISTATS}}, volume 206 of \emph{Proceedings of Machine Learning Research}, pp.\  3493--3514. {PMLR}, 2023.

\bibitem[Wang et~al.(2019)Wang, Liu, Huang, Feng, Peng, and Wang]{wang2019miasec}
Wang, C., Liu, G., Huang, H., Feng, W., Peng, K., and Wang, L.
\newblock Miasec: Enabling data indistinguishability against membership inference attacks in mlaas.
\newblock \emph{IEEE Transactions on Sustainable Computing}, 5\penalty0 (3):\penalty0 365--376, 2019.

\bibitem[Wang \& Wang(2023)Wang and Wang]{wang2023link}
Wang, X. and Wang, W.~H.
\newblock Link membership inference attacks against unsupervised graph representation learning.
\newblock In \emph{Proceedings of the 39th Annual Computer Security Applications Conference}, pp.\  477--491, 2023.

\bibitem[Wang et~al.(2021)Wang, Wang, Wang, Zhou, Liu, Bi, Ding, and Rajasekaran]{wang2020against}
Wang, Y., Wang, C., Wang, Z., Zhou, S., Liu, H., Bi, J., Ding, C., and Rajasekaran, S.
\newblock Against membership inference attack: Pruning is all you need.
\newblock In \emph{Proceedings of the Thirtieth International Joint Conference on Artificial Intelligence}, pp.\  3141--3147. ijcai.org, 2021.

\bibitem[Watson et~al.(2022)Watson, Guo, Cormode, and Sablayrolles]{watson2021importance}
Watson, L., Guo, C., Cormode, G., and Sablayrolles, A.
\newblock On the importance of difficulty calibration in membership inference attacks.
\newblock In \emph{{International Conference on Learning Representations}}, 2022.

\bibitem[Weiss et~al.(2016)Weiss, Khoshgoftaar, and Wang]{weiss2016survey}
Weiss, K., Khoshgoftaar, T.~M., and Wang, D.
\newblock A survey of transfer learning.
\newblock \emph{Journal of Big data}, 3\penalty0 (1):\penalty0 1--40, 2016.

\bibitem[Xue et~al.(2022)Xue, Yuan, He, Wu, Wu, Zhang, Liu, and Liu]{xue2022use}
Xue, M., Yuan, C., He, C., Wu, Y., Wu, Z., Zhang, Y., Liu, Z., and Liu, W.
\newblock Use the spear as a shield: An adversarial example based privacy-preserving technique against membership inference attacks.
\newblock \emph{IEEE Transactions on Emerging Topics in Computing}, 11\penalty0 (1):\penalty0 153--169, 2022.

\bibitem[Yao et~al.(2007)Yao, Rosasco, and Caponnetto]{yao2007early}
Yao, Y., Rosasco, L., and Caponnetto, A.
\newblock On early stopping in gradient descent learning.
\newblock \emph{Constructive Approximation}, 26:\penalty0 289--315, 2007.

\bibitem[Ye et~al.(2022)Ye, Maddi, Murakonda, Bindschaedler, and Shokri]{ye2022enhanced}
Ye, J., Maddi, A., Murakonda, S.~K., Bindschaedler, V., and Shokri, R.
\newblock Enhanced membership inference attacks against machine learning models.
\newblock In \emph{Proceedings of the 2022 ACM SIGSAC Conference on Computer and Communications Security}, pp.\  3093--3106, 2022.

\bibitem[Yeom et~al.(2018)Yeom, Giacomelli, Fredrikson, and Jha]{yeom2018privacy}
Yeom, S., Giacomelli, I., Fredrikson, M., and Jha, S.
\newblock Privacy risk in machine learning: Analyzing the connection to overfitting.
\newblock In \emph{31st {IEEE} Computer Security Foundations Symposium}, pp.\  268--282. IEEE, 2018.

\bibitem[Zhang et~al.(2017)Zhang, Bengio, Hardt, Recht, and Vinyals]{Zhang2017}
Zhang, C., Bengio, S., Hardt, M., Recht, B., and Vinyals, O.
\newblock Understanding deep learning requires rethinking generalization.
\newblock In \emph{5th International Conference on Learning Representations}, 2017.

\bibitem[Zhang et~al.(2021)Zhang, Xiong, and Wu]{ijcai2021p467}
Zhang, X., Xiong, H., and Wu, D.
\newblock Rethink the connections among generalization, memorization, and the spectral bias of dnns.
\newblock In Zhou, Z.-H. (ed.), \emph{Proceedings of the Thirtieth International Joint Conference on Artificial Intelligence, {IJCAI-21}}, pp.\  3392--3398. International Joint Conferences on Artificial Intelligence Organization, 8 2021.
\newblock \doi{10.24963/ijcai.2021/467}.
\newblock URL \url{https://doi.org/10.24963/ijcai.2021/467}.
\newblock Main Track.

\bibitem[Zhang et~al.(2023)Zhang, Zhao, and Wang]{zhang2023mida}
Zhang, Y., Zhao, L., and Wang, Q.
\newblock Mida: Membership inference attacks against domain adaptation.
\newblock \emph{ISA transactions}, 2023.

\bibitem[Zhang \& Sabuncu(2018)Zhang and Sabuncu]{zhang2018generalized}
Zhang, Z. and Sabuncu, M.
\newblock Generalized cross entropy loss for training deep neural networks with noisy labels.
\newblock \emph{Advances in Neural Information Processing Systems}, 31, 2018.

\bibitem[Zheng et~al.(2021)Zheng, Cao, and Wang]{zheng2021resisting}
Zheng, J., Cao, Y., and Wang, H.
\newblock Resisting membership inference attacks through knowledge distillation.
\newblock \emph{Neurocomputing}, 452:\penalty0 114--126, 2021.

\bibitem[Zou et~al.(2019)Zou, Shen, Jie, Zhang, and Liu]{zou2019sufficient}
Zou, F., Shen, L., Jie, Z., Zhang, W., and Liu, W.
\newblock A sufficient condition for convergences of adam and rmsprop.
\newblock In \emph{Proceedings of the IEEE/CVF Conference on Computer Vision and Pattern Recognition}, pp.\  11127--11135, 2019.

\end{thebibliography}

\bibliographystyle{icml2024}

\newpage
\appendix
\onecolumn



\section{Proof of Theorem~\ref{convex loss}} \label{proof for convex loss}
\begin{proof}
Since $\ell$ is strictly convex, $\ell^{\prime \prime} >0$, so there must exits infimum $B = \inf \ell^{\prime \prime}(x) \geqslant 0$.

\begin{align*} 
    \mathbb{E}_{\mathcal{D}}\ell(p_y) &=
    \mathbb{E}_{\mathcal{D}}[-\ell^{\prime}(1)(1-p_y) +\frac{1}{2} \ell^{\prime \prime}(\xi(p_y))(1-p_y)^2]  \\ 
    &\geqslant  -\ell^{\prime}(1) \mathbb{E}_{\mathcal{D}}[(1-p_y)] + \frac{B}{2} 
 \mathbb{E}_{\mathcal{D}}(1-p_y)^2 \\
    & =A\epsilon + \frac{B}{2}(\epsilon^2 + \sigma^2)
\end{align*}
where $A = -\ell^{\prime} (1)>0$, $B \geqslant 0$ is a non-negative lower bound of $\ell^{\prime \prime}(x)$.

which concludes the proof. $\hfill\square$
\end{proof}

\section{Proof of Theorem~\ref{concave loss}} \label{proof for concave loss}

\begin{proof}
 Since $\ell$ is concave, so $\ell^{\prime}(1) \leqslant \ell^{\prime}(x) < 0$.
Note that $\ell(1) =0$, then by Mean Value Theorem, we have $\ell(0) = \ell(0)-\ell(1) = - \ell^{\prime}(\xi_1) \leqslant -\ell^{\prime}(1)$.

Hence, $\ell$ is bound and continues in $[0,1]$, so $\ell^{\prime \prime}$ is also continues in $[0,1]$.


Suppose that the probability density function of $p_y$ is $f(x)$
\begin{align*} 
    \mathbb{E}_{\mathcal{D}}\ell(p_y) &=
    \mathbb{E}_{\mathcal{D}}[-\ell^{\prime}(1)(1-p_y) +\frac{1}{2} \ell^{\prime \prime}(\xi(p_y))(1-p_y)^2]  \\ 
    &= A\epsilon + \frac{1}{2} \mathbb{E}_{\mathcal{D}} [\ell^{\prime \prime}(\xi(p_y))(1-p_y)^2]  \\
    &= A\epsilon + \frac{1}{2} \int_{0}^{1} \ell^{\prime \prime}(\xi(x))f(x)(1-x)^2 dx\\
    &= A\epsilon +\frac{\ell^{\prime \prime} (\xi)}{2}\int_{0}^{1}f(x)(1-x)^2 dx \quad  \text{(By First Mean Value Theorem for Integration)} \\
    &= A\epsilon + B(\epsilon^2 + \sigma^2)
\end{align*}

which concludes the proof.

\end{proof}

\section{Connection between the loss variance and $\sigma^2$}
\label{app:variance_connection}
Notably, the metric commonly employed for MIA is typically a function of the probability vector $\mathbf{P} = (p(1|\boldsymbol{x}), p(2|\boldsymbol{x}), \dots, p(K|\boldsymbol{x}))^\top$. An example of such a metric is the entropy, given by $-\sum_{k=1}^{K}p(k|\boldsymbol{x}) \log p(k|\boldsymbol{x})$, which constitutes a mapping from $\mathbb{R}^K$ to $\mathbb{R}$. Consequently, we demonstrate that augmenting the variance $\mathrm{Var}(f(\boldsymbol{P}))$ is equivalent to amplifying $\mathrm{Var}({p_y})$ while maintaining a constant mean.

Consider $\mathbf{P} = (P_1,P_2,\dots, P_K)$ as a K-dimensional random vector, where $P_k$ represents the random variable for $p(k|\boldsymbol{x})$. Below we provide an approximate of $\mathrm{Var}(f(\boldsymbol{P}))$ by delta method \cite{oehlert1992note}.

\begin{lemma} \label{delta approx}
Let \( \mathbf{P}\) be a K-dimensional random vector in $\mathcal{P}$ and \( \boldsymbol{\mu} \) be the mean vector of \( \mathbf{P} \). Suppose \( f: \mathcal{P} \to \mathbb{R} \) is a differentiable function and \( \mathbf{J}(\boldsymbol{\mu}) \) is the Jacobian matrix of \( f \) evaluated at \( \boldsymbol{\mu} \), then the variance of \( f(\mathbf{P}) \) can be approximated by
\begin{equation}
    \mathrm{Var}(f(\mathbf{P})) \approx \mathbf{J}(\boldsymbol{\mu})^{\top} 
 \, \boldsymbol{\Sigma} \, \mathbf{J}(\boldsymbol{\mu})
\end{equation}
where $\boldsymbol{\Sigma} = \mathrm{Cov}(\mathbf{P})$ is covariance matrix of $\mathbf{P}$ and the  elements of \( \mathbf{J} \) are given by \( J_{i} = \frac{\partial f}{\partial P_i}(\boldsymbol{\mu}) \) for \( i = 1, \dots, K \).
\end{lemma}

Based on Lemma \ref{delta approx}, we can establish a relationship in variance between the confidence of true label $p_y$ and metrics used for MIA with respect to prediction probability. Assuming that the output of a neural network after a softmax layer follows Dirichlet distribution \cite{malinin2018predictive}, then we explore the connection between the loss variance and $\sigma^2$, as shown below.

\begin{proposition} \label{dirichlet}
Let $\mathbf{P}$ follows Dirichlet distribution in the simplex \( \Delta^{K-1} \), where \( \Delta^{K-1} = \left\{ (x_1, x_2, \dots, x_K) \in \mathbb{R}^K \mid x_i \geq 0, \sum_{i=1}^{K} x_i = 1 \right\} \). With a fixed $\mathbb{E} \mathbf{P} = \boldsymbol{\mu}$, for any $i \in \mathcal{Y}$, an increase in \( \mathrm{Var}(P_i) \) implies a corresponding increase in \( \mathrm{Var}(f(\mathbf{P}))\)
\end{proposition}

\begin{proof}
Suppose $\mathbf{P_1} \sim \mathrm{Dirichlet}(\alpha_1, \alpha_2, \dots, \alpha_K)$ and $\mathbf{P_2} \sim \mathrm{Dirichlet}(\beta_1, \beta_2, \dots, \beta_K)$ such that $\mathbb{E} \mathbf{P_1} = \mathbb{E} \mathbf{P_2} = \boldsymbol{\mu}$ and $\mathrm{Var}(P_{2,t}) > \mathrm{Var}(P_{1,t})$, where $P_{i,j}$ is the $j$-th element of $\mathbf{P_i}$ and $t \in {\mathcal{Y}}$.

By $\mathbb{E} \mathbf{P_1} = \mathbb{E} \mathbf{P_2}$, we have 

\begin{equation*}
    \mathbb{E} [P_{1,j}] = \frac{\alpha_j}{\alpha_0} = 
    \frac{\beta_j}{\beta_0} = \mathbb{E} [P_{2,j}], \quad \text{for any $j=1,2,\dots, K$}
\end{equation*}
where $\alpha_0 = \sum_{j=1}^{K} \alpha_j$ and $\beta_0 = \sum_{j=1}^{K} \beta_j$.

By the condition $\mathrm{Var}(P_{2,t}) > \mathrm{Var}(P_{1,t})$, we have

\begin{align*}
        \mathrm{Var}(P_{2,t}) > \mathrm{Var}(P_{1,t})  
      \Longrightarrow    \frac{(\tilde{\alpha}_t)(1-\tilde{\alpha}_t)}{\alpha_0 +1 }>
        \frac{(\tilde{\beta}_t)(1-\tilde{\beta}_t)}{\beta_0+1} 
        \Longrightarrow  \beta_0 > \alpha_0
\end{align*}

where $\tilde{\alpha}_t = \frac{\alpha_t}{\alpha_0} = \frac{\beta_t}{\beta_0} =\tilde{\beta}_t$.

It follows that 
\begin{equation*}
    \mathbf{\Sigma}^{(1)}_{ij} =  \mathrm{Cov}(P_{1,i} , P_{1,j}) = \frac{-\tilde{\alpha}_i \tilde{\alpha}_j } {\alpha_0 +1} > \frac{-\tilde{\beta}_i\tilde{\beta}_j}{\beta_0+1} = \mathrm{Cov}(P_{2,i} , P_{2,j}) = \mathbf{\Sigma}^{(2)}_{ij}, \quad \text{for any $i \ne j$}
\end{equation*}

\begin{align*}
    \mathrm{Var}(P_{2,j}) > \mathrm{Var}(P_{1,j}), \quad \text{for any $j=1,2,\dots, K$} 
\end{align*}

This implies

\begin{align*}
    \mathbf{J}(\boldsymbol{\mu})^{\top} 
 \, \boldsymbol{\Sigma}^{(1)} \, \mathbf{J}(\boldsymbol{\mu}) 
  &= \mathbf{J}(\boldsymbol{\mu})^{\top} 
 \,  [ \frac{\beta_0 +1}{\alpha_0 +1}  \boldsymbol{\Sigma}^{(2)} ] \, \mathbf{J}(\boldsymbol{\mu}) \\
 &> \mathbf{J}(\boldsymbol{\mu})^{\top} 
 \,   \boldsymbol{\Sigma}^{(2)} \, \mathbf{J}(\boldsymbol{\mu}) 
\end{align*}

 By Lemma \ref{delta approx}, we have $\mathrm{Var}(f(\mathbf{P})) \approx \mathbf{J}(\boldsymbol{\mu})^{\top} 
 \, \boldsymbol{\Sigma} \, \mathbf{J}(\boldsymbol{\mu})$, so we can conclude the proof. $\hfill\square$

\end{proof}

According to Proposition \ref{dirichlet}, any metric related to \( f(\mathbf{P}) \)—including entropy and modified-entropy \cite{song2021systematic}—will exhibit increased variance in response to a rise in the variance of \( p_y \). Building upon this insight and in conjunction with \ref{metric adv}, it can be deduced that amplifying the variance of \( p_y \) during the training phase can enhance the model's resilience against metric-based MIA.

\section{Proof of Lemma \ref{R1} and Lemma \ref{R2}} 
\label{lipschitz}
\subsection{Proof of Lemma \ref{R1}}
\begin{proof}
Given that $A(x)$ and $B(x)$ are Lipschitz continuous, for all $x, y \in D$, we have

\begin{align}
\Vert A(x) - A(y)\Vert &\leq L_A \Vert x - y\Vert \newline
\Vert B(x) - B(y)\Vert &\leq L_B \Vert x - y\Vert
\end{align}

Consider the product $C(x) = A(x)B(x)$. For any $x, y \in D$,

$$
\begin{aligned}
\Vert C(x) - C(y)\Vert &= \Vert A(x)B(x) - A(y)B(y)\Vert  \\
&= \Vert A(x)B(x) - A(x)B(y) + A(x)B(y) - A(y)B(y)\Vert \\
&= \Vert A(x)(B(x) - B(y)) + (A(x) - A(y))B(y)\Vert \\
&\leqslant  \Vert A(x)\Vert \cdot \Vert B(x) - B(y)\Vert + \Vert A(x) - A(y)\Vert \cdot \Vert B(y)\Vert \\
&\leqslant  M_A L_B \Vert x - y\Vert + L_A M_B \Vert x - y\Vert \\
&=  (M_A L_B + L_A M_B)\Vert x - y\Vert
\end{aligned}
$$

Therefore, proving that $C(x) = A(x)B(x)$ is Lipschitz continuous with a Lipschitz constant $L_C = M_A L_B + L_A M_B$, under the given conditions of boundedness and Lipschitz continuity of $A(x)$ and $B(x)$.
\end{proof}

\subsection{Proof of Lemma \ref{R2}}
Since $f$ and $g$ are Lipschitz continuous, we have for all $x, y \in \mathbb{R}^n$,

$$\Vert f(x) - f(y)\Vert \leq L_f \Vert x - y\Vert,$$

and for all $u, v \in \mathbb{R}^m$,

$$\Vert g(u) - g(v)\Vert \leq L_g \Vert u - v\Vert.$$

Consider two points $x, y \in \mathbb{R}^n$. We want to show that $h$ is Lipschitz continuous, i.e., there exists a constant $L_h$ such that

$$\Vert h(x) - h(y)\Vert = \Vert g(f(x)) - g(f(y))\Vert \leq L_h \Vert x - y\Vert.$$

Using the Lipschitz continuity of $g$ and then $f$, we have

$$\Vert g(f(x)) - g(f(y))\Vert \leq L_g \Vert f(x) - f(y)\Vert \leq L_g L_f \Vert x - y\Vert.$$

Setting $L_h = L_g L_f$, we see that

$$\Vert h(x) - h(y)\Vert \leq L_h \Vert x - y\Vert,$$

proving that the composition $h = g \circ f$ is Lipschitz continuous with Lipschitz constant $L_h = L_g L_f$.

\section{Can CCL improve other convex functions?} \label{other convex functions}

Our method defined in Equation \ref{ConcaveLoss} just integrates CE as the convex loss function. There we show that our method can also improve other convex loss functions such as Focal Loss \cite{lin2017focal}. 

\begin{equation*}
\ell = \alpha \ell_{\mathrm{FL}} + (1- \alpha) \tilde{\ell} 
\end{equation*}
where $\ell_{\mathrm{FL}}$ is Focal Loss and $\tilde{\ell} \in \mathcal{F}$.

\begin{figure*}
    \centering
    \includegraphics[width=0.6\linewidth]{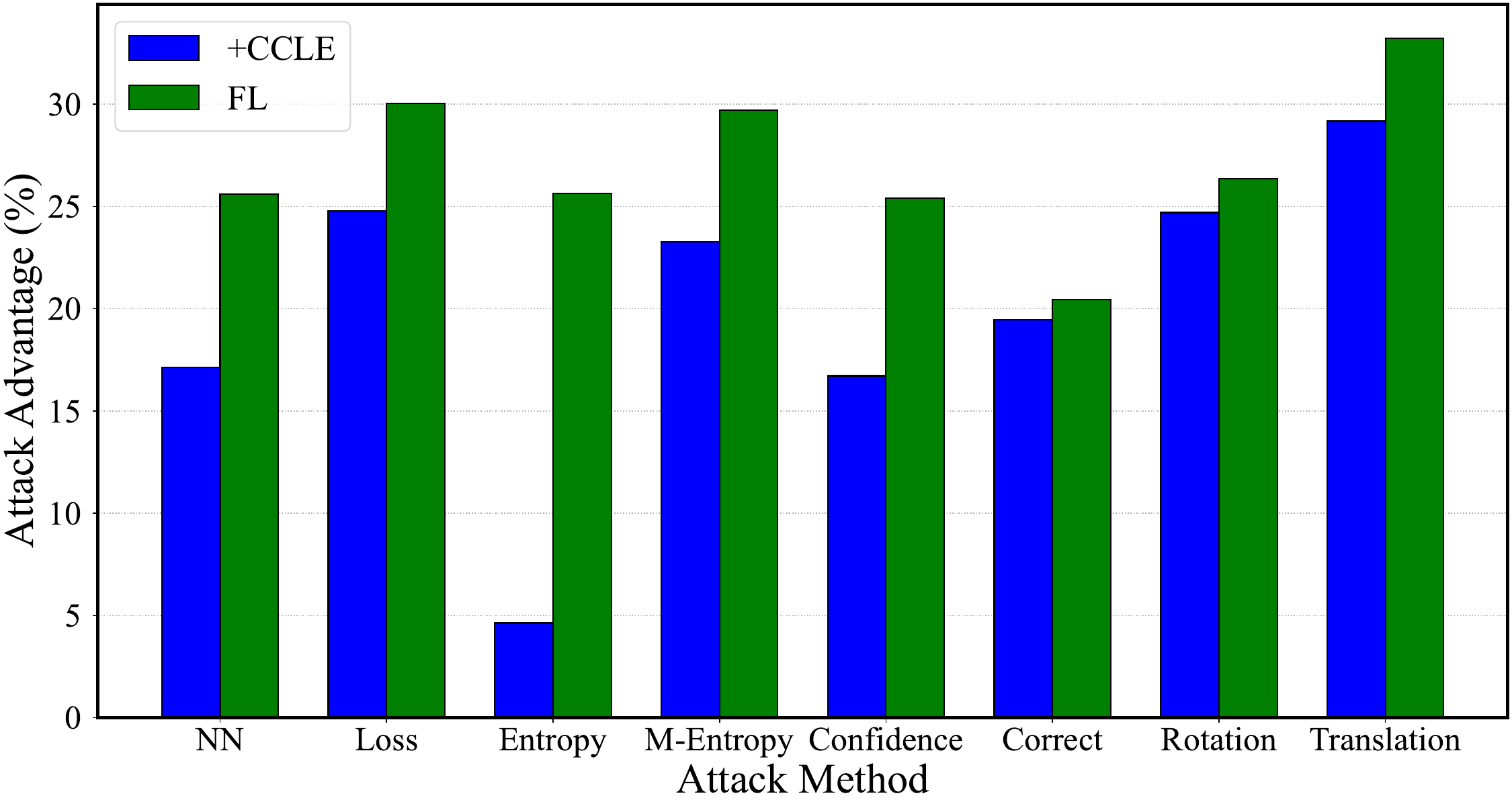}
    \caption{Attack advantage comparison among focal loss (FL) and focal loss equipped with CLE (+CCLE) on CIFAR-10.} 
    \label{fig:cifar10_focal}
\end{figure*}

In the experiments, we use Resnet-34 trained on CIFAR-10. For focal loss, we fixed $\gamma=2$. In particular, we use CCLE and select the best $\alpha$ in \{0.2,0.4,0.6,0.8\} with restricted condition that performs better utility. Our results in Figure \ref{fig:cifar10_focal} show that focal loss equipped with a concave term helps defend MIA across eight attack methods.

\section{Defense Methods with Hyperparameter} \label{sec_exp_setting}


\subsection{Other Defense Methods}
\paragraph{RelaxLoss.} RelexLoss \cite{chen2022relaxloss} reduces the gap between the member and non-member loss distribution by applying gradient ascent as long as the average loss of the current batch is smaller than $\alpha$. We vary the $\alpha$ over \{0.01, 0.04, 0.1, 0.2, 0.4, 0.8, 1.6, 3.2 \}.

\paragraph{Mixup+MMD.}  Mixup+MMD integrates the Maximum Mean Discrepancy (MMD) approach with mix-up training techniques. Specifically, MMD serves as a metric for quantifying the divergence between two empirical data distributions, functioning as follows
$$
\text{Distance}(X, Y) = \left\| \frac{1}{n} \sum_{i=1}^{n} \phi(x_i) - \frac{1}{m} \sum_{j=1}^{m} \phi(y_j) \right\|_H
$$
where $\phi(\dot)$ is Gaussian kernel function. MMD regularization loss is calculated by a batch of training and validation instances. We vary the weight of the MMD term across \{0.01, 0.02, 0.05, 0.1, 0.2, 0.5, 1, 2, 4, 8\}.

\paragraph{Adversary Regularization.} Adversary regularization \cite{nasr2018machine} conducts a min-max game optimization and an adversarial training algorithm that minimizes classification loss while also reducing the maximum gain of potential membership inference attacks. We vary the weight of the adversarial loss over \{0.8,1.0,1.2,1.4,1.6,1.8\}.

\paragraph{Dropout.}  Dropout randomly deactivates a subset of neurons in a layer with a given probability during training. In our experiments, dropout is applied specifically to the last fully connected layer of each target model. We vary the probability of dropout over \{0.01, 0.02, 0.05, 0.1, 0.3, 0.5, 0.7, 0.9\}.

\paragraph{Label Smoothing.} Label smoothing \cite{guo2017calibration} modifies the target labels, making them slightly less confident by replacing the hard 0 and 1 targets with values slightly closer to a uniform distribution. We vary the smoothing parameter over \{0.1, 0.2, 0.3, 0.4, 0.5, 0.6, 0.7, 0.8, 0.9 \}.

\paragraph{Confidence Penalty.}  Confidence penalty \cite{pereyra2017regularizing} mitigates overfitting by penalizing low entropy in the output distributions of neural networks. In particular, it is implemented by adding an entropy regularization term to the objective. We vary the weight of the regularization term over \{0.1, 0.3, 0. 5, 1, 2, 4, 8\}.

\paragraph{Early Stopping.} Early Stopping \cite{yao2007early} monitors the model's performance on a validation set and stops the training process when performance begins to degrade. Following the implementation of \citet{chen2022relaxloss}, our experiments save checkpoints at specific epochs: \{25, 50, 75, 100, 125, 150, 175, 200, 225, 250, 275\}

\end{document}